%% file: main.tex
\documentclass[11pt]{article}

\usepackage{fullpage}
\usepackage{natbib}
\usepackage{algorithm}
\usepackage{algorithmic}

\usepackage{amsthm}
\usepackage{amssymb}

\newtheorem{theorem}{Theorem}
\newtheorem{lemma}{Lemma}

\newcommand{{\ours}}{FedNew}

\usepackage{microtype}
\usepackage{graphicx}
\usepackage{amsfonts}
\usepackage{amsmath}
\usepackage{subcaption}
\usepackage{booktabs} 
\usepackage{xcolor}
\usepackage{hyperref}
\usepackage{hyperref}

\usepackage{bm}

\newcommand{\R}{\mathbb{R}}
\newcommand{\eqdef}{\; { := }\;}

\DeclareMathOperator*{\argmin}{arg\,min}

\title{\bf {\ours}: A Communication-Efficient and Privacy-Preserving  Newton-Type Method for Federated Learning}

\author{Anis Elgabli\thanks{University of Oulu, correspondent author: anis.elgabli@oulu.fi}, Chaouki Ben Issaid\thanks{University of Oulu}, Amrit S. Bedi\thanks{University of Marylad, College Park, MD, USA.},\\ Ketan Rajawat\thanks{Indian Institute of Technology Kanpur, India.}, Mehdi Bennis\thanks{University of Oulu}, and Vaneet~Aggarwal\thanks{School of Industrial Engineering and School of Electrical and Computer Engineering, Purdue University, USA.} 
}

\begin{document}

\maketitle

\begin{abstract}
Newton-type methods are popular in federated learning due to their fast convergence. Still, they suffer from two main issues, namely: \emph{low communication efficiency} and \emph{low privacy} due to the requirement of sending Hessian information from clients to parameter server (PS). In this work, we introduced a novel framework called FedNew in which there is no need to transmit Hessian information from clients to PS, hence resolving the bottleneck to improve communication efficiency. In addition, FedNew hides the gradient information and results in a privacy-preserving approach compared to the existing state-of-the-art. 
The core novel idea in {\ours} is to introduce a two level framework, and alternate between updating the inverse Hessian-gradient product using only one alternating direction method of multipliers (ADMM) step and then performing the global model update using Newton’s method. Though only one ADMM pass is used to approximate the inverse Hessian-gradient product at each iteration, we develop a novel theoretical approach to show the converging behavior of {\ours} for convex problems. Additionally, a significant reduction in communication overhead is achieved by utilizing stochastic quantization. Numerical results using real datasets show the superiority of FedNew compared to existing methods in terms of communication costs.
\end{abstract}

\section{Introduction} \label{introduction}
In this paper, we consider the following {\em federated learning} (FL) problem
\begin{align}\label{erm-prob}
\min\limits_{x\in \mathbb{R}^d} f(x) \eqdef \frac{1}{n}\sum \limits_{i=1}^n f_i(x),
\end{align}
where $d$ denotes the dimension of the model $x\in\R^d$ we wish to train, $n$ is the total number of clients/clients in the system, $f_i(x)$ is a convex loss/risk associated with data stored on client $i\in[n] \eqdef \{1, 2, \dots, n\}$, and $f(x)$ is the empirical loss/risk.
While first-order methods for FL problems are extensively studied in the literature, recently a handful of second-order methods have been proposed for FL problems. The main advantage of second-order methods is their faster convergence rate, although they suffer from high communication costs. Besides that, sharing the first and second-order information (which contains client's data) may create a privacy issue. This is because the Hessian matrix contains valuable information about the properties of the local function/data, and when shared with the parameter server (PS), privacy may be violated by eavesdroppers or an honest-but-curious PS. For instance, authors in \citep{Yin2014} show that the eigenvalues of the Hessian matrix can be used to extract important information of the input images.

In this work, we are interested in developing communication efficient  second-order methods while still preserving the privacy of the individual clients. To this end, we propose a novel framework that hides the gradients as well as the Hessians of the local functions, yet uses second-order information to solve the FL problem. In particular, we divide the standard Newton step into \emph{outer} and \emph{inner} levels. The objective of the inner level is to learn what we call the {\em Hessian inverse-gradient product} or {\em Zeroth Hessian inverse-gradient product} for the computation-efficient version (i.e., $(\nabla^2 f(x^k))^{-1} \nabla f(x^k)$ or $(\nabla^2 f(x^0))^{-1} \nabla f(x^k)$) without sharing individual gradients and/or Hessians. We use the ADMM method to solve the inner level problem due to their faster convergence as compared to gradient descent (GD) methods. Specifically, one ADMM step is used to approximate the solution of the inner level problem at each iteration. The solution of the inner level is used to perform the  outer level update, which is similar to Newton's method. The proposed approach ensures privacy and communication efficiency. Next, we summarize the state-of-the-art of first and second-order FL methods. 
\subsection{First-order FL Methods}
The {\em de-facto} solution to problem \eqref{erm-prob} is to use distributed (stochastic) gradient descent method (DGD/DSGD). In the $k$-th iteration of DGD, the PS shares $x^k$ with all clients, each client computes its local gradient with respect to $x^k$ and transmits it to the PS which aggregates them to form the global gradient and perform one GD step as follows.
\begin{align}\label{defacto_sol}
x^{k+1} =x^k - \alpha \nabla f(x^k) = x^k  -  \frac{\alpha}{n}\sum_{i=1}^n\nabla f_i(x^k).
\end{align}
 Throughout this paper, we will refer to  FL with DGD as FedGD. To enable communication-efficient FL, several techniques were proposed for enhancing communication efficiency, which is a key bottleneck in the FL. These techniques include {\em model compression/quantization} \citep{FEDLEARN, QSGD, DCGD, QGADMM},  {\em local computation} utilizing schemes such as local SGD \citep{FL2017-AISTATS, Stich-localSGD, localSGD-AISTATS2020, FedRR},  and  {\em censoring or partial participation} \citep{FL2017-AISTATS, SGD-AS, LAG}. Moreover, solutions that utilize a combination of the above techniques (e.g., censoring and quantization) were also investigated \citep{LAQ, CQGADMM}. Other relevant techniques for further reducing the communication cost of FL methods include the use of {\em momentum} \citep{DIANA, ADIANA}, and {\em adaptive learning rates} \citep{MM2019, xie2019local, reddi2020adaptive, IntSGD}. 	 
\subsection{Towards Second-order FL Methods}
First-order FL methods suffer from slow convergence speed in terms of the number of iterations/communication rounds. Moreover, their convergence speed is a function of the condition number, which depends on: (i) the structure of the model being trained, (ii) the choice of loss function, and (iii) the distribution of training data. On the other hand, second-order methods are known to be much faster owing to the fact that they make an extra computational effort to estimate the local curvature of the loss landscape, which is useful in generating faster and more adaptive update directions. Therefore, second-order methods perform more computations per communication round to achieve less number of communication rounds. Since in FL, it is often communication and not computation that is the key bottleneck, second-order methods are becoming attractive and have recently gained attention. The Newton's method, which forms the basis for most efficient second-order methods,  enjoys a fast {\em condition-number-independent} (local) convergence rate \citep{Beck-book-nonlinear}, is given by
\begin{align}\label{newtonUpdate} 
x^{k+1}
&=x^k-\Big(\sum_{i=1}^n\nabla^2 f_i(x^k)\Big)^{-1} \sum_{i=1}^n\nabla f_i(x^k).
\end{align}
However, fewer number of communication rounds does not necessarily induce a lower communication cost. Communication overhead is also affected by the number of transmitted bits between the clients and PS, which depends on the size of the transmitted vector/matrix per communication round. Hence, a direct application of Newton's method does not induce an efficient distributed implementation as it requires repeated communication of the local Hessian matrices $\nabla^2 f_i \in \R^{d\times d}$ to the server. This is prohibitive and constitutes a massive number of transmitted bits,  which requires high communication energy and bandwidth. Another important concern when implementing Newton's method is privacy, as it requires sharing both the gradient and the Hessian at each iteration, which makes it vulnerable to inversion {attacks} \cite{fredrikson2015model,hitaj2017deep}. For instance, in a linear regression task, the Hessian matrix is nothing but $D^TD$, where $D$ is the data matrix which results in privacy issues for each client. 


Recently, \cite{FedNL} proposed a Newton-based framework that avoids communicating  the full Hessian matrix at each iteration. The idea is to share a compressed version of the Hessian matrix utilizing compression techniques such as top-K, and Rank-R approximation. While this approach can reduce the communication cost, it has a number of shortcomings: (i) it does not solve the privacy issue since every client still shares the gradient and a compressed version of the Hessian matrix, (ii) it introduces further computations to perform compression, and (iii) it may not lead to high communication savings when the rank of the Hessian matrix is high. It is also worth mentioning that this approach, as well as the standard Newton's method, require matrix inversion at the PS at each iteration. On the other hand, even though the Newton Zero algorithm (proposed in \cite{FedNL})  requires matrix inversion only at the first iteration, it still shares the full Hessian matrix at the initial iteration, which necessitates $\mathcal{O}(d^2)$ communication cost at the first iteration, besides overlooking privacy. Moreover, in contrast to first-order FL methods which rely on simple aggregation at the PS, existing Newton-based methods require matrix inversion and multiplication (at least for the first iteration) which restricts the use of quantization to Hessian and gradient at the same time. 

To obviate the above-mentioned limitations, we propose {\ours}, a novel and efficient framework that ensures privacy by hiding the Hessian and the gradient. i.e., neither gradient nor Hessian is sent directly. In addition, {\ours} reduces the communication cost compared to standard second-order methods by transmitting only $\mathcal{O}(d)$ information at each iteration, similar to first-order methods. Furthermore, {\ours} reduces the computation cost when utilizing the zeroth Hessian at each iteration (similar to Newton Zero), since it performs matrix inversion only once (at the first iteration). Finally, to further reduce communication overhead, we propose a variant, coined Q-{\ours}, which quantizes the transmitted variable. 
 %
%
We summarize our contributions as follows.
 
(1) We propose a novel framework to solve the FL problem using second-order information. In particular, we decompose the objective function of the problem of learning the {\em Hessian inverse-gradient product} into a sum of separable functions and solve it in a distributed way. The framework alternates between updating the {\em Hessian inverse-gradient product} and the global model. To the best of our knowledge, this is the first work that utilizes one step of the ADMM method to estimate Newton directions with convergence guarantees. 
 
(2) The proposed {\ours}  algorithm is rooted in a communication-efficient and privacy-preserving way to estimate the {\em Hessian inverse-gradient product} at each iteration. In contrast to existing approaches, {\ours} does not require clients to share their gradient and the Hessian matrix (or compressed Hessian matrix) at any iteration, including the first one; hence $\mathcal{O}(d)$ communication cost is guaranteed at each iteration including the first one. 
 
(3) We prove that the proposed FedNew algorithm asymptotically converges to the optimal direction of Newton method. We develop a novel proof technique to show that the proposed ADMM based inner and outer level iterates interacts with each other and shows converging behavior under some assumptions (cf. Sec. \ref{convergence}). Moreover, we provide a privacy analysis of FedNew and show that the reconstruction of gradients/Hessians is not possible (cf. Sec. \ref{privacy}).

(4) To further reduce the communication cost per iteration, we leverage stochastic quantization and propose quantized version called Q-FedNew. It was possible due to the unique feature of FedNew where clients share only a vector with PS that is not involved in any further multiplication/inversion at the PS (cf. Sec. \ref{squantization}).
 
 
 We conduct extensive simulations on real datasets and show the performance gain of the proposed  {\ours} and Q-{\ours} algorithms in comparison to Newton Zero and FedGD in terms of the number of transmitted bits and number of communication rounds.
 
 
 \section{Proposed Framework: {\ours}}
We start by writing the Newton update step for the global model $x$ at iteration $k$ as
\begin{equation}\label{newtonUpdate}
x^{k+1}=x^k-(\nabla^2 f(x^k))^{-1} \nabla f(x^k),
\end{equation}
where $f$ (the empirical loss function) is assumed to be a continuously differentiable function over $x\in \mathbb{R}^d$. 
A key observation is that the direction $z(x^k)$$:=$$(\nabla^2 f(x^k))^{-1} \nabla f(x^k)$ is the solution of the following problem
\begin{align}\label{innerProb}
z(x^k) = \argmin_{y \in \mathbb{R}^d}  \frac{1}{2}y^T \nabla^2 f({x^k}) y - y^T \nabla f(x^k).
\end{align}
Note that the problem in \eqref{innerProb} is an unconstrained convex optimization problem that can be solved at iteration $k$ to obtain the optimal direction $z(x^k)$. To solve this problem, one could utilize any existing iterative solver and obtain a direction that is close to the optimal $z(x^k)$ for a given $x^k$. 
Utilizing the  solution of \eqref{innerProb}, we can then update $x$ as, $x^{k+1}=x^k-z(x^k)$ which is a one Newton step from $x^{k}$. 
However, the main issue here is that solving \eqref{innerProb} iteratively introduces a double loop, which can be communication and computation expensive. To address this issue, we propose a two-level framework detailed next. 

\subsection{Inner Level: One Pass ADMM}
We note that solving \eqref{innerProb} completely at each iteration $k$ exhibits three challenges: (i) it requires sharing local Hessians and gradients with the PS which is communication expensive, (ii) it creates privacy issues since clients share their gradients and Hessians, and (iii) it requires matrix inversion at the PS at each iteration. We will explain how to overcome the limitation (iii) later in this section. Let us first start by addressing the limitations in (i) and (ii), by reformulating the problem as follows
\begin{align}
&\underset{y_i, y}{\min} \quad \frac{1}{n}\sum_{i=1}^n \Big(\frac{1}{2}y_i^T (\nabla^2 f_i(x^{k} )+\alpha I) y_i - y_i^T \nabla f_i(x^k )\Big) \nonumber \\
&\quad\text{s.t. }  \quad y_i = y, \quad \forall i \in [n],
\label{decentralizedAdmm}
\end{align}
where we introduce a tuning parameter $\alpha$. The problem in \eqref{decentralizedAdmm} is a consensus reformulation of the problem in \eqref{innerProb} where $y_i$ denotes the local directions, and $y$ denotes the global direction. Note that the direction $-[\nabla^2 f(x^k)+\alpha I]^{-1} \nabla f(x^k)${, the inexact Newton direction,} is also a descent direction \cite{marteau2019globally, karimireddy2018global, mishchenko2021regularized,zhang2021newton} which boils down to Levenberg-Marquardt algorithm for least square problems which exhibits global convergence \cite{levenberg1944method,marquardt1963algorithm,bergou2020convergence}. 
By reformulating \eqref{innerProb} as \eqref{decentralizedAdmm}, the objective function becomes separable across clients, which allows the solution to be distributed and avoid clients sharing their Hessians/gradients at each iteration. We leverage ADMM algorithm to solve the problem in \eqref{decentralizedAdmm} in a distributed manner. 

To simplify the notation, from now on, we will use $H_i^{k}$ for Hessian $\nabla^2 f_i(x^k)$ and $g_i^k$ for gradient $\nabla f_i(x^k)$ of client $i$ at iteration $k$. With that, the augmented Lagrangian of the optimization problem \eqref{decentralizedAdmm} can be written as,
\begin{align}
 \mathcal{L}_{\rho}\left(\{y_i\}_{i=1}^n, {y}, \lambda \right)\!=\!&\sum_{i=1}^n \Big(\frac{1}{2}y_i^T (H_i^k\!+\!\alpha I)  y_i - y_i^T g_i^k\Big) + \sum_{i=1}^{n} \langle \lambda_i, y_{i} - y \rangle+\frac{\rho}{2}\sum_{i=1}^{n} \|y_{i} - y\|_2^2,
\label{augmentedLag4}
\end{align}
where $\lambda = \{\lambda_i\}_{i=1}^n$ is the collection of dual variables and $\rho > 0$ is a constant penalty parameter. One ADMM pass is performed at each iteration $k$. Hence, the primal and the dual variables are updated as follows.
\begin{itemize}
\item[$(1)$] Each client $i$ updates its primal variable $y_i^k$ by solving the following problem,
\begin{align}\label{updateLocal}
y_i^k\!=\!\underset{y_i}{\arg \min}~ \Big\{&\frac{1}{2}{y_i}^T (H_i^k+\alpha I) y_i \!+\!\!\langle \lambda_i^{k-1}\!, y_i\!-\!y^{k-1} \rangle-\! {y_i}^T g_i^k\!\!+\frac{\rho}{2} \|y_i - y^{k-1}\|_2^2\Big\},
\end{align}
which gives the solution, 
\begin{align}\label{updateLocal2}
y_i^k\!=\! (H_i^k+\alpha I+\rho I)^{-1}(g_i^k-\lambda_i^{k-1}+\rho y^{k-1}).
\end{align}
Then, every client transmits its updated local variable $y_i^k$ to the PS.

\item[$(2)$] The primal variable of the PS is updated by solving the following problem
\begin{align}\label{updatePS}
y^{k}\!=\!\underset{y}{\arg \min}~ \Big\{&\sum_{i=1}^n\langle \lambda_i^{k-1}, y_i^{k} - y \rangle \!\!+\!\!\frac{\rho}{2}\sum_{i=1}^n \|y_i^k - y\|_2^2\Big\},
\end{align}
Which gives the solution,
\begin{align}\label{updatePS2}
y^k\!&=\! \frac{1}{n} \sum_{i=1}^n (y_i^k+\lambda_i^{k-1}/\rho).
\end{align}
Once the global variable $y^k$ is updated at the PS, it will be shared with all clients.
\item[$(3)$] The dual variables are updated locally for each client
\begin{align}\label{updateDual}
\lambda_i^k= \lambda_i^{k-1} + \rho (y_i^k-y^k).
\end{align}
\end{itemize}
From \eqref{updatePS2} and \eqref{updateDual}, we know that $\sum_{i=1}^n \lambda_i^{k} = 0, \forall k$, hence \eqref{updatePS2} can be written as
\begin{align}\label{updatePS3}
y^k\!&=\frac{1}{n} \sum_{i=1}^n y_i^k.
\end{align}
Therefore, the global variable $y^k$ is just the average of the local variables $y_i^k$, for all $i \in [n]$.
\subsection{Outer Level: Approximate Newton Step}
At the outer level, after calculating the global direction $y^k$ from \eqref{updatePS3}, we perform the following update 
\begin{align}\label{newtonOurs}
x^{k+1}=x^k- y^k,
\end{align}
where $y^k$ is an approximation to the optimal direction $y^*(x^k)$. 
We note that the closed-form expression in \eqref{updateLocal2} still involves an inversion of the local Hessian $H_i^{k}$, which can be avoided by using $H_i^{0}$ instead of $H_i^{k}$ at each iteration. To summarize, at iteration $k$, given $g_i^{k}$ and $H_i^{k}$ ($H_i^{0}$ if we avoid inversion), each client updates and transmits $y_i^{k}$. Then the PS updates $y^{k}$, and performs a one Newton step before sharing the updated model and $y^{k}$ with all the clients. Finally, each client updates the dual variables via \eqref{updateDual}. The detailed steps of the algorithm are summarized in Algorithm \ref{alg:FedNew}. Using ADMM to solve the inner level optimization problem introduces a new set of variables (local primal and dual variables), and by performing only one ADMM pass each time, the subproblems are not solved optimally, which may affect the convergence of the outer iterate $x^k$ in our framework. Therefore, it becomes essential to study the convergence behavior of the inner iterates so that outer iterates eventually follows the newton directions only. Proving convergence for such a combination between a primal based method (Newton method) for the outer problem and a primal-dual based method with one pass only each iteration  (one-pass ADMM method) for the inner problem is very challenging. However, in the next section, we provide convergence analysis of the proposed framework and show that convergence holds under some assumptions and conditions.   
\begin{algorithm}[t]
\caption{{\ours} (\textbf{Fed}erated \textbf{New}ton) }
\label{alg:FedNew}
\begin{algorithmic}[1]
\STATE \textbf{Parameters:} $K$; $\rho$, $\alpha$ 
\STATE \textbf{Initialization:} $x^0, y^0, \{y_i^0\}_{i=1}^n, \{\lambda_i^0\}_{i=1}^n \in\R^d$
\STATE $k \leftarrow 0$
\WHILE{$k < K$}
\STATE \textbf{on all clients:} compute local gradient $ g_i^k $ and local Hessian $ H_i^k $.
\STATE \textbf{on all clients:} update $y_i^k$ using \eqref{updateLocal2} and send to the PS
\STATE \textbf{on the PS:} update $\boldsymbol{y}^k$ using \eqref{updatePS3} and $x^k$ using \eqref{newtonOurs} and transmit them to all clients
\STATE  \textbf{on all clients:} update $\lambda_i^k$ using \eqref{updateDual}
\ENDWHILE
\end{algorithmic}
\end{algorithm}
\section{Convergence Analysis}\label{convergence}
In this section, we study the convergence of the proposed FedNew algorithm. Note that we have introduced a two level framework to approximate the update in \eqref{newtonUpdate}. The solution of \eqref{innerProb} would result in the optimal direction, which we then use to update the model parameter at the PS. Since \eqref{innerProb} is strictly convex, there exists a unique optimal solution which we can write in closed form, but it requires that all clients transmit their hessian matrices to the PS, where the PS performs summation of the local Hessians, and then matrix inversion, which is costly, and we want to avoid. The path we took is to use an iterative distributed ADMM algorithm to solve the inner level problem in \eqref{decentralizedAdmm}. There are two possibilities, the first is that we solve the inner problem iteratively till convergence, and then utilize the solution for the Newton update. In this case, $x^{k}$ would trivially converge to $x^*$ for $\alpha=0$ following the existing analysis in the literature \cite{karimireddy2018global,marteau2019globally}, but it would add a lot of computational burden on the clients and is not very practical solution to aim for. In contrast, we propose to solve the inner problem via one step ADMM (where we just take one step of standard ADMM), and then perform the outer Newton update. Note that this approach is really effective for practical use, but introduces errors in the Newton directions to be used for the outer update. This makes the convergence proof challenging and does not hold straightforwardly from the existing proof of the ADMM or Newton based methods. To address this challenge, we develop a proof technique next to show that FedNew converges asymptotically.  

To analyze the convergence behavior of Algorithm \ref{alg:FedNew}, we assume that each function $f_i$ is twice differentiable, convex, has $L$-Lipschitz continuous gradient, and $L_{*}$-Lipschitz continuous Hessian. Let $Q_i(x,y)$ be equal to $\frac{1}{2}y_i^T (\nabla^2 f_i(x^{k} )+\alpha I) y_i - y_i^T \nabla f_i(x^k )$, with this definition, we assume that $\nabla Q$ is Lipschitz continuous in $y$ with constant $L_q$, i.e., for any given $x^1$, $x^2$ $\in \bm{X}$, we have
\begin{align}
    \|\nabla Q_i(x^1,y_i^1)-\nabla Q_i(x^2,y_i^2)\|^2\leq L_q \|y_i^1-y_i^2\|^2.
\end{align}
We start with the necessary and sufficient optimality conditions of the inner problem in \eqref{decentralizedAdmm} at the $k$-th iteration (the $k$-th Newton step), which are the primal and dual feasibility conditions \citep{boyd2011distributed} defined as
\begin{align}
&\!\!{y_i^\star}(x^k) \!=\! {y^\star}(x^k),  \quad\quad\quad\quad\quad\text{(primal feas.)} \label{primal_feasiblity} 
\\
&\!\!(H_i^k\!+\!\alpha I){y_i^\star}(x^k)\!-\!g_i^k \!+\! {\lambda_i^\star}(x^k) =0, \quad\text{(dual feas.)} \label{dual_feasibility}
\end{align}
{for all $i \in [n]$. In \eqref{primal_feasiblity}-\eqref{dual_feasibility}, ${y_i^\star}(x^k)$ and ${\lambda_i^\star}(x^k)$ denote the optimal values of $y_i^k$ and $\lambda_i^k$, respectively, at the $k^{th}$ iteration, i.e. when running the ADMM steps to the end. Note that if ${y_i^\star}(x^k)$ is computed then ${x}^{k+1}=x^k - {\frac{1}{n}\sum_{i=1}^n y_i^\star}^k$}. Next, we write the optimality conditions at the optimal model $x^*$ as 
%
%
%
\begin{align}\label{primal_feasiblity2}
y_i^\star =& y^\star = 0, \forall i \in [n]  &\!\!\!\!\!\text{(primal feas.)}
\\
\label{dual_feasibility2}
g_i^\star =& \lambda_{i}^\star, \forall i \in [n]&\!\!\!\!\!\!\text{(dual feas.)}
\end{align}
where $g_i^k$ becomes $g_i^\star$ in \eqref{dual_feasibility2}. The equality in \eqref{primal_feasiblity2} and \eqref{dual_feasibility2} follows from the fact that $y^\star=\nabla^2 [f({x}^{\star})+\alpha I]^{-1}\nabla f({x}^{\star})=0$ because $\nabla f({x}^{\star})=0$. Note that we drop the $k$ notation in \eqref{primal_feasiblity2}-\eqref{dual_feasibility2} because we write them for $x^*$ which is independent of $k$. Next, to simplify the analysis, we introduce a local copy $x_i$ of the model $x$ to write
\begin{align}\label{eq:a0}
x^{k+1} = \frac{1}{n}\sum_{i=1}^n x_i^{k+1} =  \frac{1}{n}\sum_{i=1}^n (x_i^k - y_i^k).
\end{align}
Re-writing \eqref{newtonOurs} as \eqref{eq:a0} does not change anything in the algorithm but simplifies the analysis. With this, we first introduce two intermediate lemmas (Lemma \ref{lemma_1} and Lemma \ref{lemma_2}) which would lead us to Theorem \ref{theorem_1}. 
\begin{lemma}\label{lemma_1}
In  Algorithm \ref{alg:FedNew} (FedNew), for each iteration $k$, it holds that
\begin{align}\label{eqo70}
\sum_{i=1}^n\langle \lambda_i^k + s^k-{\lambda_i^\star}^k, y_i^k-{y^\star}^k\rangle \leq -\alpha \sum_{i=1}^n\|y_i^k-{y^\star}^k\|^2,
\end{align}
where $s^k=\rho(y^{k}-y^{k-1})$ is the dual residual of the inner problem.
\end{lemma}
%
%
%
%
The proof of Lemma \ref{lemma_1} is provided in Appendix \ref{proof_lemma_1}. The result in Lemma \ref{lemma_1} will be used in Lemma \ref{lemma_2} to impose an upper bound on the difference between the optimality gap of the inner problem at iteration $k$ and $k-1$.
\begin{lemma}\label{lemma_2}
In Algorithm \ref{alg:FedNew} (FedNew), for each iteration $k$, it holds that
\begin{align}\label{lemma2Result}
&\frac{1}{\rho} \sum_{i=1}^n \|\lambda_i^k- {\lambda_i^\star}^k\|^2+2\beta_1 \sum_{i=1}^n \|y_i^{k}-{y^\star}^k\|^2  + \rho n \|y^{k}-{y^\star}^k \|^2\nonumber\\
&+2\rho n \|y^{k} - y^{k-1}\|^2
\nonumber\\
&
 \leq \frac{1}{\rho} \sum_{i=1}^n \|\lambda_i^{k-1} -{\lambda_i^\star}^{k-1}\|^2+\frac{2 L_q^2}{\rho}\sum_{i=1}^n \|y_i^{k-1}-{y^\star}^{k-1}\|^2\nonumber\\
 &\quad +\frac{4 L_q^2n}{\rho}\|y^{k-1}-{y^\star}^{k-1}\|^2+2\rho n\|y^{k-1}-y^{k-2}\|^2\nonumber\\
 &\quad -2\beta_2 \sum_{i=1}^n \|y_i^{k}-{y^\star}^k\|^2,
\end{align}
where $\beta_1 > 0, \beta_2 > 0$, and 
\begin{align}\label{alphaCondition}
\beta_1+\beta_2 \leq \alpha - 2.5\rho - \frac{8L_q^2n}{\rho}.
\end{align}

\end{lemma}

%
%
The proof of Lemma \ref{lemma_2} is provided in Appendix \ref{proof_lemma_2}. 
Now, we are ready to state the main theoretical result of this work. To proceed towards the main theorem, we define Lyapunov function $V^k$ as
\begin{align}
V^k\!:=&\frac{1}{\rho} \sum_{i=1}^n \|\lambda_i^k- {\lambda_i^\star}^k\|^2+2\beta_1 \sum_{i=1}^n \|y_i^{k}-{y^\star}^k\|^2 
\nonumber
\\
&+ \rho n \|y^{k}-{y^\star}^k \|^2+2\rho n \|y^{k} - y^{k-1}\|^2,\label{Lyapunov}
\end{align}
which quantifies the distance to the optimal for the dual variable $\lambda_i^k$, and the primal variable $y^k$. We present the main result in Theorem \ref{theorem_1}. 

\begin{theorem}\label{theorem_1}
With $\alpha$ that satisfies \eqref{alphaCondition} for $\beta_1\geq \frac{L_q^2}{\rho}$, $\rho\geq 2L_q$, and $\beta_2 > 0$, the sequence of iterates of {\ours} (cf. Algorithm \ref{alg:FedNew}) are such that {$\lambda_i^k \rightarrow {\lambda_i^\star}^k$, $y_i^{k} \rightarrow {y^\star}^k$, and $y^{k} \rightarrow {y^\star}^k$ as $k \rightarrow \infty$.}
\end{theorem}
%
%
\begin{proof}
The detailed proof of Theorem \ref{theorem_1} is provided in Appendix \ref{proof_of_theorem}. To prove the asymptotic convergence of the proposed algorithm, we utilize the results of Lemma \ref{lemma_1} and Lemma \ref{lemma_2}, and show that the Lyapunov function $V^k$ is monotonically decreasing for each $k$. Subsequently, we utilize the Monotone Convergence Theorem to claim that $V^k$ converges to zero as $k\rightarrow \infty$. This further implies that $\lambda_i^k \rightarrow {\lambda_i^\star}^k$, $y_i^{k} \rightarrow {y^\star}^k$, and $y^{k} \rightarrow {y^\star}^k$ as $k \rightarrow \infty$.
\end{proof}
%
%
%
%
{We remark that even though we obtain the approximate inexact Newton directions via approximating the solution of the distributed optimization problem formulated in \eqref{decentralizedAdmm} at each iteration, we will asymptotically move in the inexact Newton direction, which is a descent direction. One can utilize the global convergence analysis of inexact Newton methods in \cite{karimireddy2018global,marteau2019globally} to further establish the global convergence of FedNew and we leave that to future scope of this work. } Finally, it is worth mentioning that replacing $H_i^k$ with $H_i^0$ in the formulation \eqref{decentralizedAdmm} and Algorithm \ref{alg:FedNew} does not change the results of Lemmas \ref{lemma_1} and Lemma \ref{lemma_2} since they both require positive definiteness of the Hessian, which already is a property of $H_i^0$. With this, it is easy to show that Theorem \ref{theorem_1} result also holds for the computation-efficient implementation of {\ours}.

\section{Privacy Analysis}\label{privacy}
Revealing the local gradient is vulnerable to model inversion, and reconstruction attacks \cite{fredrikson2015model,hitaj2017deep}. In addition, revealing the Hessian increases vulnerability, since more information about the local function/data is released. These attacks infer the statistical profiles of training samples and violate data privacy. Against such an adversarial inverse problem, we aim at preserving privacy defined as follows.

{\bf Definition 1 \cite{zhang2018admm}}\quad A mechanism $M:$ $M(X)$$ \rightarrow Y$ is defined to be privacy preserving if the input~$X$ cannot be uniquely derived from the output $Y$.

We treat $X$ as local gradient/Hessian to be protected, and consider $Y$ as the known information at an eavesdropper such as the PS or another client. At iteration k, under the standard Newton's method, the PS receives the gradient $g_i^k$ and the Hessian $H_i^k$ from each client $i$, thereby violating the privacy defined in Definition 1. In sharp contrast, the PS in {\ours} receives $y_i^k$, which is neither the gradient nor the Hessian. Consequently, we ensure that privacy is preserved against curious PS and against any eavesdropper with the knowledge of the updating trajectory of the variable $y$. The following theorem  formally states that {\ours} preserves the privacy of the local gradients/Hessians.
\begin{theorem}\label{privacyAnalysisTheorem1}
 At each iteration $k \geq 0$, {\ours} preserves the privacy of 
 each local gradient update $g_i^{k}$ and local zero-Hessian $H_i^k$.
\end{theorem}
\begin{proof}
The eavesdropper needs to solve \eqref{updateLocal2} with respect to $H_i^k$ and $g_i^k$. However, this single equation has three unkowns at the reciever which are $H_i^k$, $g_i^k$, and $\lambda_i^{k-1}$. Hence, the receiver, the eavesdropper cannot have a unique solution for $H_i^k$, $g_i^k$ since the number of variables $V=3$ is greater than the number of equations $E=1$. This finalizes the proof.
\end{proof}
The key point of the proof is to show that the inverse problem of an eavesdropper boils down to solving a set of equations at every iteration, in which the number of unknowns is larger than the number of equations. Therefore, each client's local gradient/Hessian cannot be uniquely derived.

\section{Quantized {\ours}}\label{squantization}
%
%
Similar to first-order methods, {\ours} allows each client to transmit a vector whose size is equivalent to the model size. At the PS, all received vectors need only to be aggregated and averaged. Therefore, in contrast to other second-order methods, the received information from all clients at the PS is not involved in any further multiplication and/or inversion. Note that such multiplication and/or inversion may make existing quantization schemes used in first-order methods no longer applicable. Hence, by utilizing this feature of {\ours}, we are able to further quantize the transmitted variable ($y_i^k$) using quantization schemes used in first-order methods in the literature \cite{QGADMM}, and numerically show the convergence of the proposed approach. With quantization, we can significantly reduce the communication overhead per iteration. We refer to the quantized version of {\ours} by Q-{\ours}.

Next, we describe the quantization procedure of Q-{\ours}. At iteration $k$,  each client $i$ quantizes the difference between $y_i^k$ and the previously quantized vector $\hat{y}_i^{k-1}$ as $y_i^k-\hat{y}_i^{k-1} = Q_i(y_i^k, \hat{y}_i^{k-1})$. The function {$Q_i(\cdot)$} is a stochastic quantization operator that depends on the quantization probability $p_{i,j}^k$ for each element $j\in\{1,2,\cdots,d\}$, and on $b_i^k$ the bits used to representing each element. We choose $p_{i,j}^k$ and $b_n^k$ as follows. 
The $j$-th element {$[\hat{y}_i^{k-1}]_j$} of the previously quantized model vector is centered at the quantization range {$2 R_i^k$} that is equally divided into $2^{b_i^k}-1$~quantization levels, yielding the quantization step size $\Delta_i^k=2 R_i^k/(2^{b_i^k}-1)$. In this coordinate, the difference between {$[y_i^k]_j$} and {$[\hat{y}_i^{k-1}]_j$} is
\begin{align}\label{quantization}
[c_i(y_i^k)]_j\!=\! \frac{1}{\Delta_i^k} \left([y_i^k]_j-[\hat{y}_i^{k-1}]_j\!+\!R_i^k\right)\!,
\end{align}
where adding $R_i^k$~ensures the non-negativity of the quantized value. Then, {$[c_i(y_i^k)]_j$} is mapped to $[q_n(y_i^k)]_j$ as
\begin{align}
\!\!\!\![q_n(y_i^k)]_j =\begin{cases}
\left\lceil [c_i(y_i^k)]_j\!\right\rceil & \text{with prob. $p_{i,j}^k$}\\[2pt]
\left\lfloor [c_i(y_i^k)]_j\!\right\rfloor & \text{with prob. $1-p_{i,j}^k$} \label{Eq:quant}
\end{cases}
\end{align}
where $\lceil\cdot \rceil$ and $\lfloor\cdot \rfloor$ are ceiling and floor functions, respectively. 
Next, we select the probability $p_{i,j}^k$ in \eqref{Eq:quant} such that the expected quantization error is $\mathbb{E}{\left[\epsilon_{i,j}^k\right]}$ is zero which implies that  
%
\begin{align}\label{mean_zero}
\nonumber &p_{i,j}^k \left( [c_i(y_i^k)]_j-\lceil[c_i(y_i^k)]_j \rceil \right)\\
 &+(1-p_{i,j}^k)  \left( [c_i(y_i^k)]_j - \lfloor[c_i(y_i^k)]_j \rfloor\right)=0.
\end{align}
Solving \eqref{mean_zero} for $p_{i,j}^k$, we obtain

%
\begin{align}
p_{i,j}^k = \left([c_i(y_i^k)]_j-\lfloor[c_i(y_i^k)]_j \rfloor \right).
\label{Eq:Optp}
\end{align} 
The $p_{i,j}^k$ selection in \eqref{Eq:Optp} ensures that the quantization error is unbiased, yielding the quantization error variance $\mathbb{E}\left[\left(\epsilon_{i,j}^k\right)^2\right] \leq({\Delta_i^k})^2/4$ \citep{reisizadeh2019exact}. This implies that $\mathbb{E}\left[\|\epsilon_{n}^k\|^2\right]\leq d({\Delta_i^k})^2/4$. 
%
%
With the aforementioned stochastic quantization procedure, $b_i^k$, $R_i^k$, and $q_i(y_i^k)$ suffice to represent $\hat{y}_i^k$, where
\begin{align}
q_i(y_i^k)=( [q_i(y_i^k)]_1,[y_i^k)]_2,\cdots,[y_i^k)]_d )^\intercal,
\end{align}
are transmitted to the PS. After receiving these values at the PS, $\hat{y}_i^k$~can be reconstructed as follows:
\begin{align}
\hat{y}_i^k = \hat{y}_i^{k-1}+ \Delta_i^k q_i(y_i^k)-R_i^k\mathbf{1}.
 \label{recoverEq}
 \end{align}
\noindent In contrast to full arithmetic precision based communication which uses $32d$ bits to represent the transmitted vector, every transmission payload size of Q-{\ours} is $b_i^k d + b_R$ bits, where $b_R\leq 32$ is the required bits to represent $R_i^k$. 
\section{Experiments}\label{experiment}
In this section, we empirically investigate the performance of the proposed algorithms {\ours} and Q-{\ours} against FedGD \citep{mcmahan2017communication}, and Newton Zero \citep{FedNL} for a binary classification problem using regularized logistic regression\footnote{The code is avaialble at \url{https://github.com/aelgabli/FedNew}.}. We consider 3 variants of {\ours}. To explain these variants, we let $r$ be the update rate of the hessian matrix. We consider $r\in \{0, 0.1, 1\}$. i.e., $r=0$, reflects the case when the hessian matrix is not updated at all. Therefore, $H_i^0$ is used at each iteration similar to Newton zero. $r=0.1$ reflects the case when the hessian is updated every $10$th iteration. Finally, $r=1$ reflects the case when the hessian is updated at each iteration, so $H_i^k$ is used at iteration $k$. As we will show later in the section, updating the hessian matrix at each iteration improves the convergence speed, but at the cost of more computations. We first describe the experimental setup and then discuss the numerical results.    

\subsection{Experimental Setup}
In our experiments, we consider the regularized logistic regression problem
\begin{equation}\label{prob:log-reg}	\min\limits_{x\in\R^d}\left\{f(x)\eqdef \frac{1}{n}\sum\limits_{i=1}^n f_i(x) +\frac{\mu}{2}\|x\|^2\right\},
\end{equation}
where the local loss functions are defined as
\begin{equation}
f_i(x) = \frac{1}{m}\sum \limits_{j=1}^m\log\left(1+\exp(-b_{ij}a_{ij}^\top x)\right),
\end{equation}
$\{a_{ij},b_{ij}\}_{j\in [m]}$ forms the data samples of the $i^{\text{th}}$ client, and $\mu \geq 0$ is a regularization parameter chosen to be equal to $10^{-3}$ in all experiments conducted in this section. We use four standard datasets taken from the LibSVM library \cite{chang2011libsvm}: a1a, w7a, w8a, and phishing. More details on each dataset, including the number of independent features and the number of clients considered, are summarized in Table \ref{table1}, where $N = m \times n$ is the total number of samples. In our experiments, each dataset is evenly split between the clients. Note that we have chosen different numbers of clients for each of the datasets to show the performance of the proposed approach under various network sizes.  For Q-{\ours}, the quantization resolution is 3 bits in all experiments. In the experiments, we plot the optimality gap $f(x^k)- f(x^\star)$ as a function of the number of communication rounds or the number of communicated bits per client, where $f(x^\star)$ is the function value at the $30^{\text{th}}$ iterate of the standard Newton’s method. Finally, for each variant of our algorithm, we choose $\alpha$ and $\rho$ that give the fastest convergence in the tested range. We would like to mention that though we theoretically prove convergence for $\alpha$ that satisfies \eqref{alphaCondition}, we observe that empirically FedNew converges for any choice of $\alpha \geq 0$. 
%
%
%

\subsection{Comparison to Baselines}
In Fig. \ref{fig1}, we plot the optimality gap as a function of the number of communication rounds for {\ours} and the two baselines: FedGD and Newton Zero. Fig. \ref{fig1} shows that {\ours}-($r=1$) is the fastest to converge compared to the other algorithms in terms of the number of communication rounds, followed by {\ours}-($r=0.1$), then both Newton Zero and {\ours}-($r=0$), and finally FedGD. In conclusion, when using the updated hessian at each iteration, FedNew can achieve significant reduction in terms of the number of communication rounds, but at the cost of more computations. On the other hand, {\ours}-($r=0$) which avoids updating the hessian can match the convergence speed of Newton Zero while preserving privacy. Periodic update of the hessian ($r=0.1$) achieves a middle point. In fact, as shown in Fig. \ref{fig1}, in some of the datasets, it converges as fast as the exact hessian based FedNwe ($r=1$) while reducing the computation cost $10$ times since the hessian is computed once every $10$ iterations.    

\begin{figure*}[h]
\centering
\begin{subfigure}{.23\textwidth}
  \centering
  \includegraphics[scale=0.3]{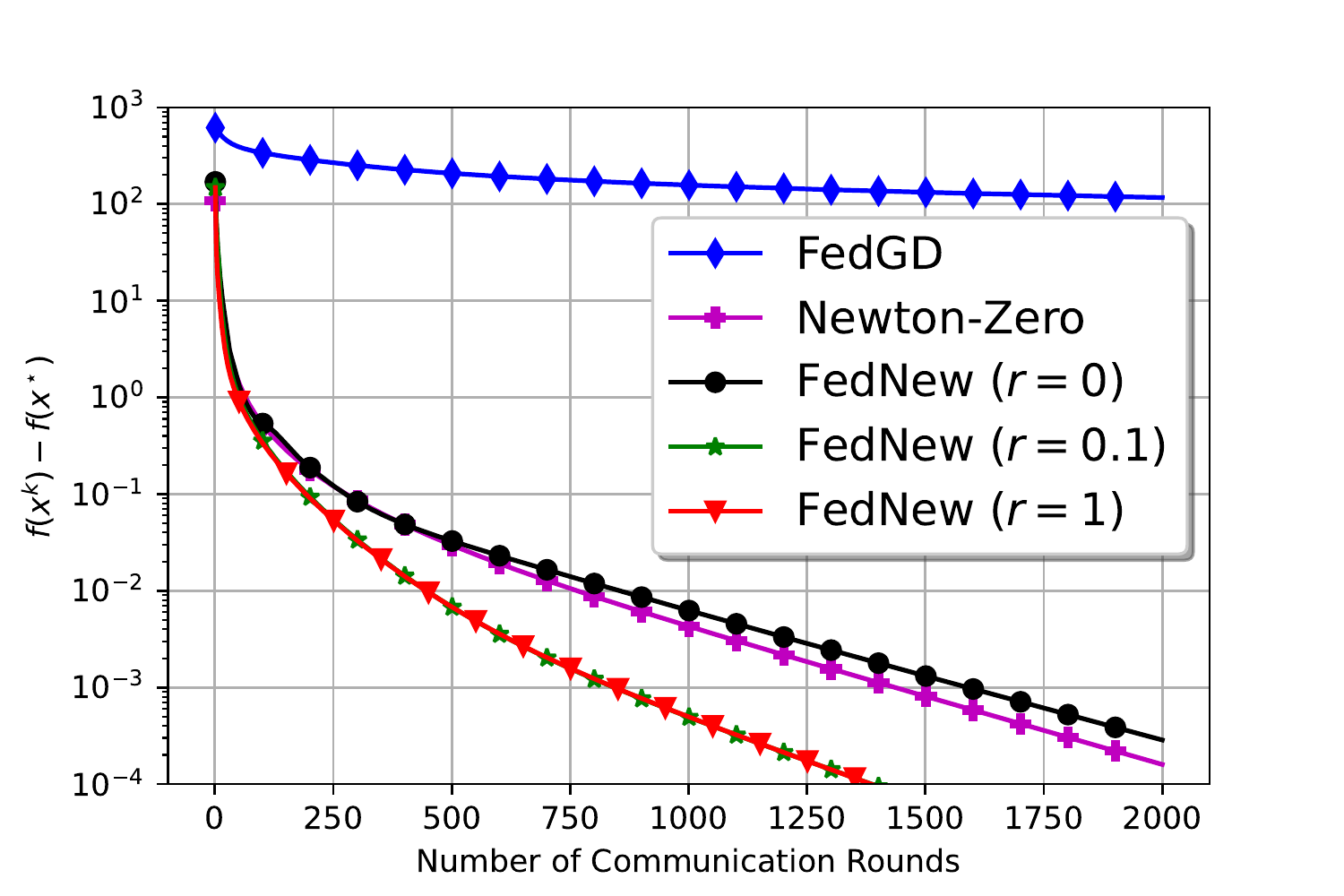}  
  \caption{Dataset: a1a}
\end{subfigure}
\begin{subfigure}{.23\textwidth}
  \centering
  \includegraphics[scale=0.3]{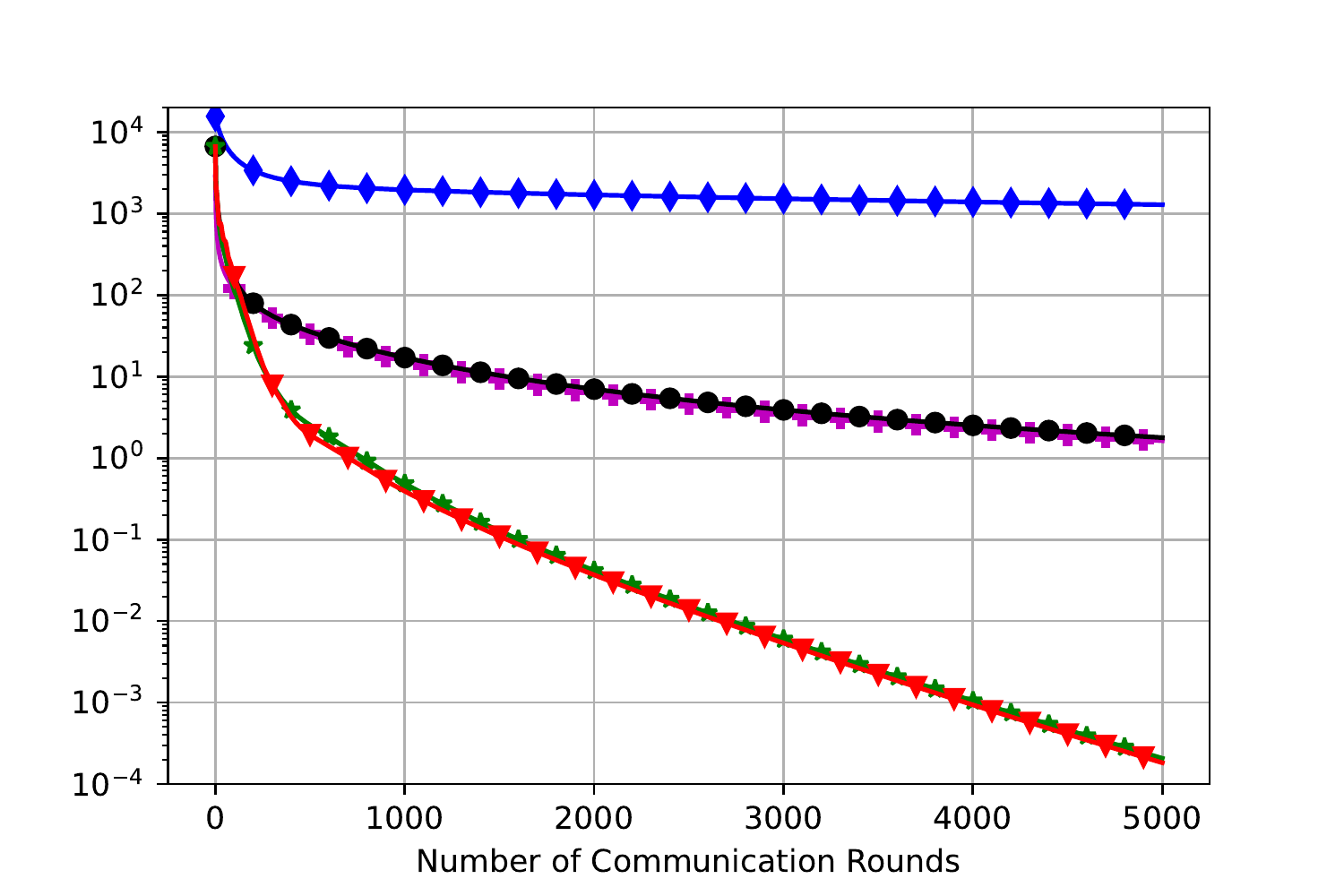}  
  \caption{Dataset: w7a}
\end{subfigure}
\begin{subfigure}{.23\textwidth}
  \centering
  \includegraphics[scale=0.3]{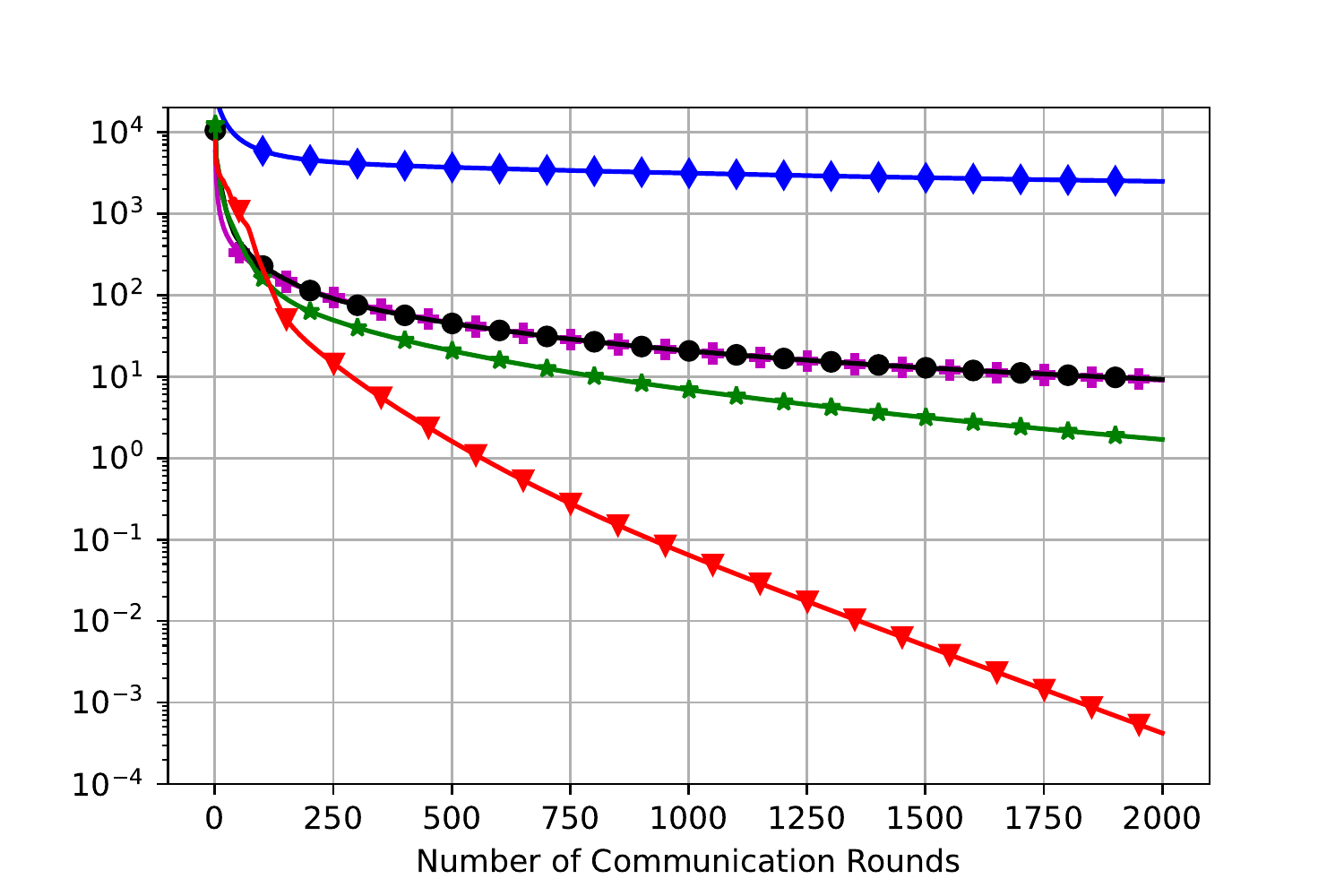} 
  \caption{Dataset: w8a}
\end{subfigure}
\begin{subfigure}{.23\textwidth}
  \centering
  \includegraphics[scale=0.3]{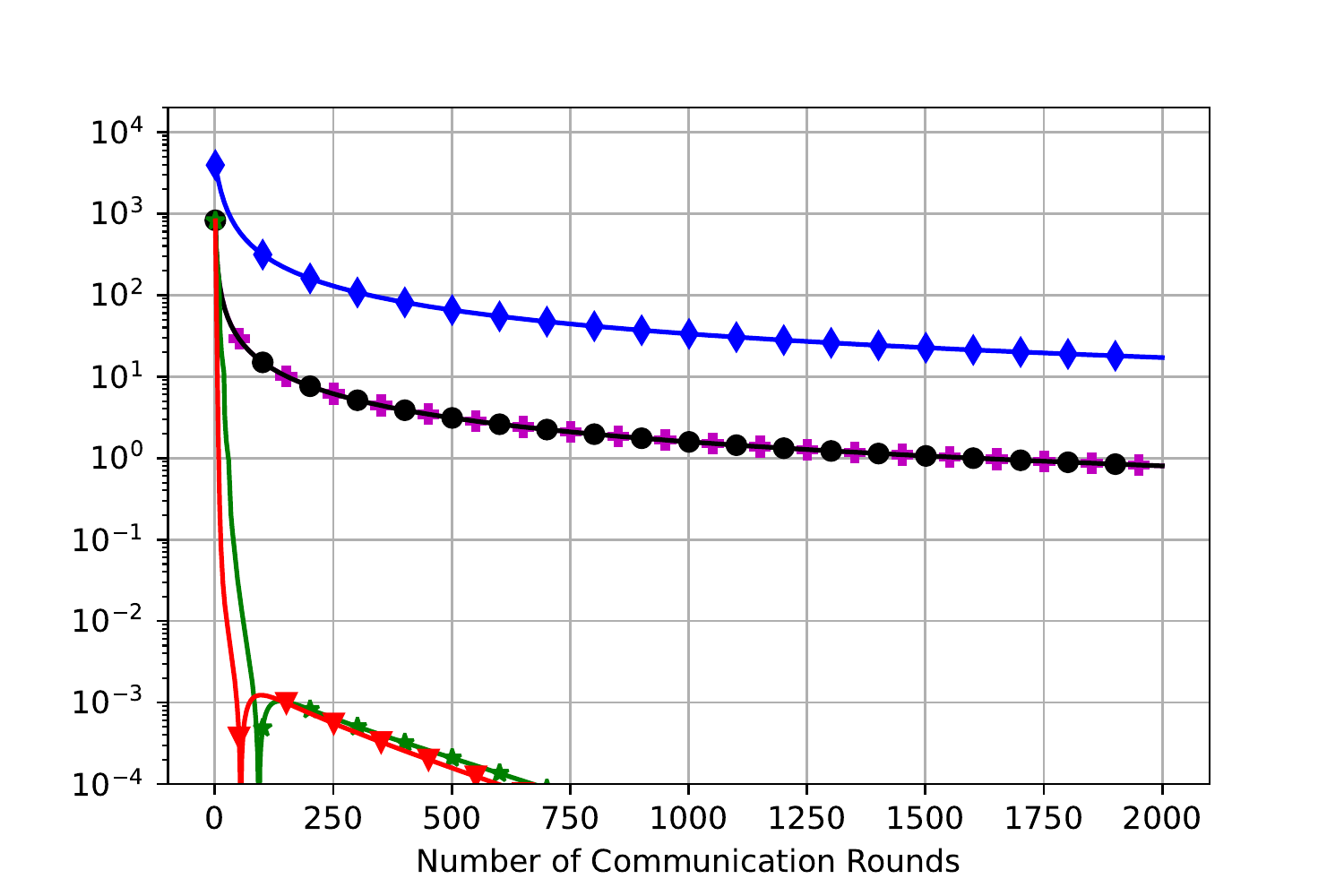} 
  \caption{Dataset: phishing}
\end{subfigure}
\caption{Optimality gap of {\ours} compared to FedGD and Newton Zero in terms of the number of communication rounds per client for different datasets. {\ours}($r=0$) require close to Newton Zero's number of communication rounds for $\epsilon$ optimality gap, yet preserves the privacy. {\ours}($r=0.1$) and {\ours}($r=1$) achieve faster convergence.}
\label{fig1}
\end{figure*}

\begin{figure*}[h]
\centering
\begin{subfigure}{.23\textwidth}
  \centering
\includegraphics[scale=0.3]{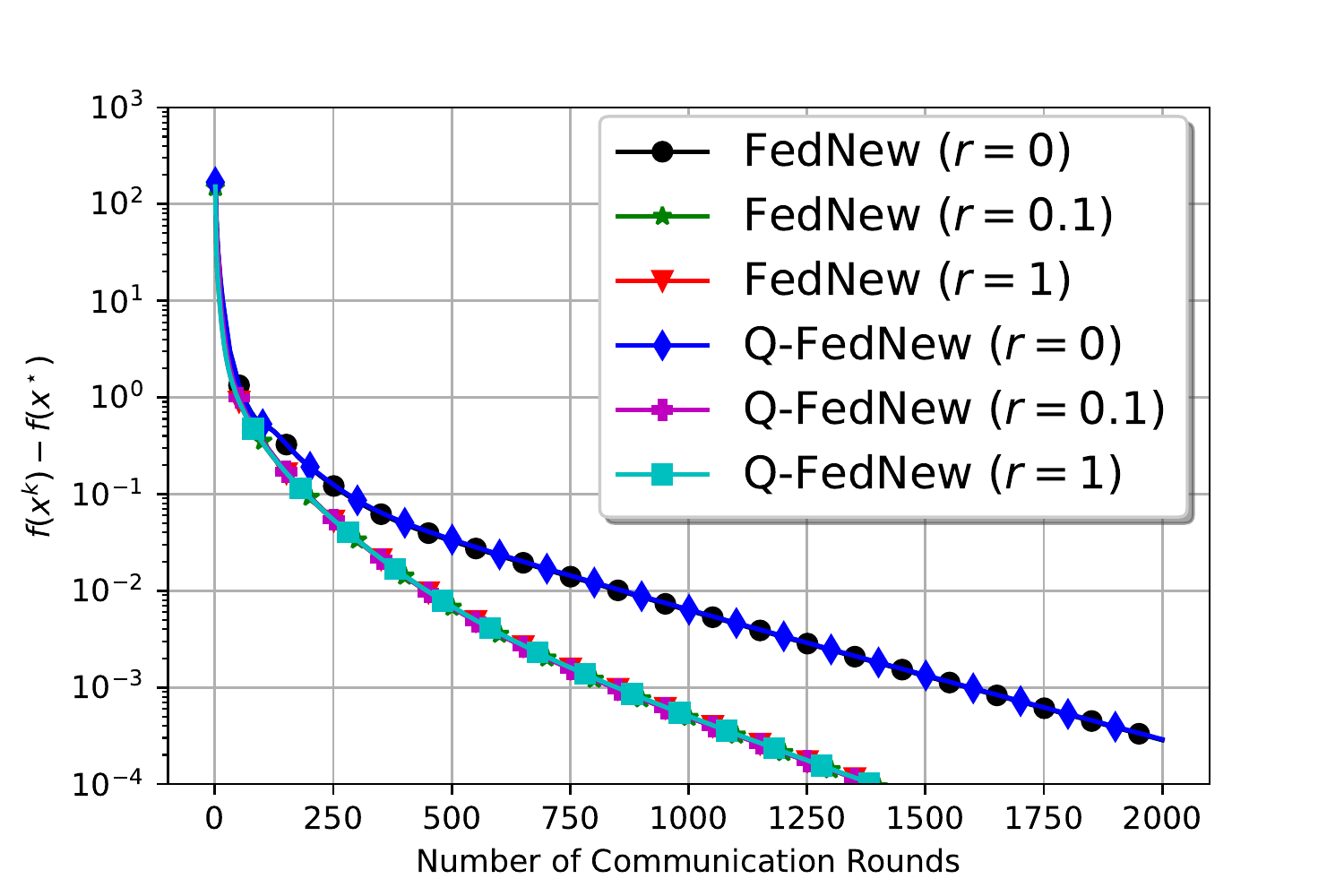} 
\end{subfigure}
\begin{subfigure}{.23\textwidth}
  \centering
\includegraphics[scale=0.3]{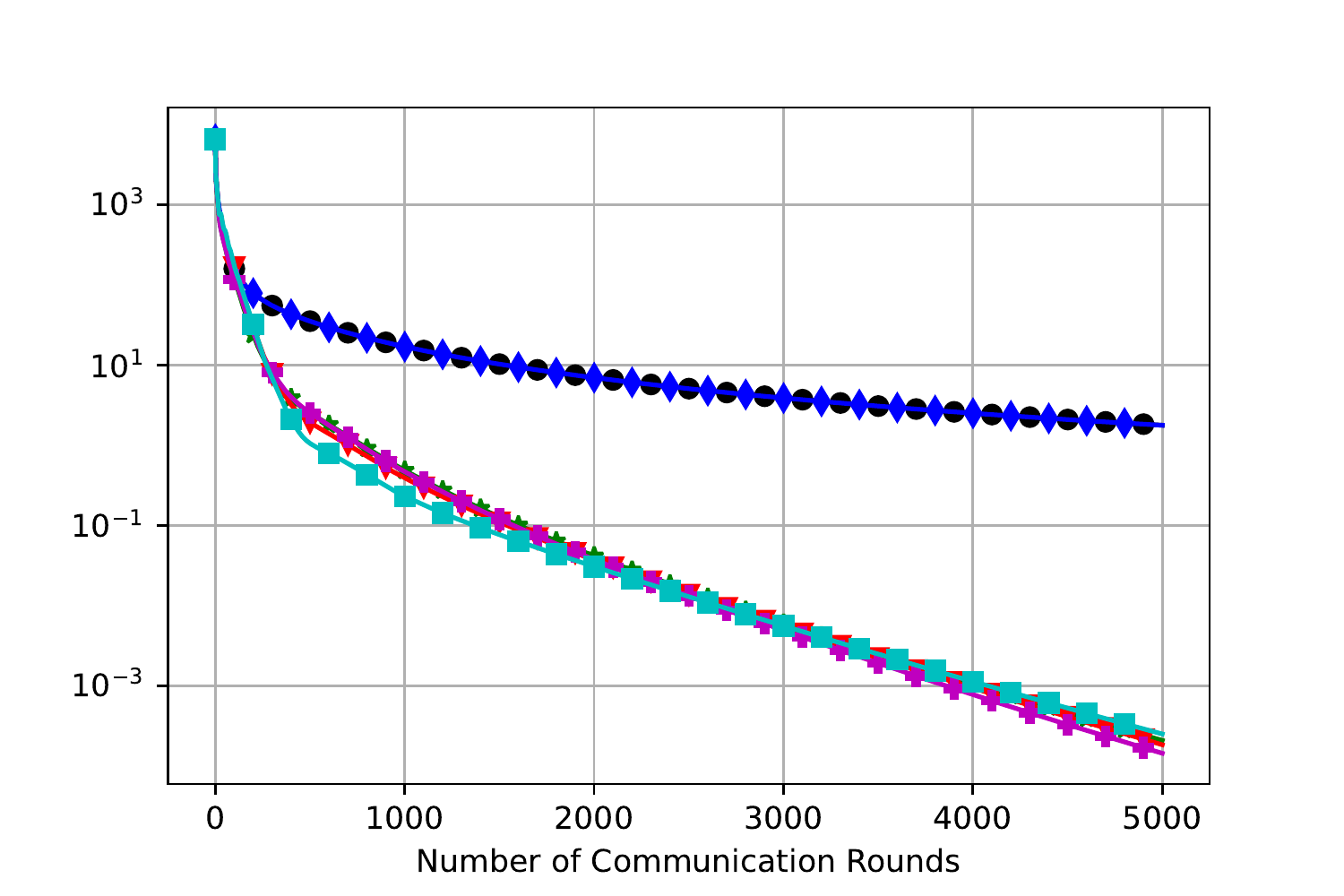} 
\end{subfigure}
\begin{subfigure}{.23\textwidth}
  \centering
\includegraphics[scale=0.3]{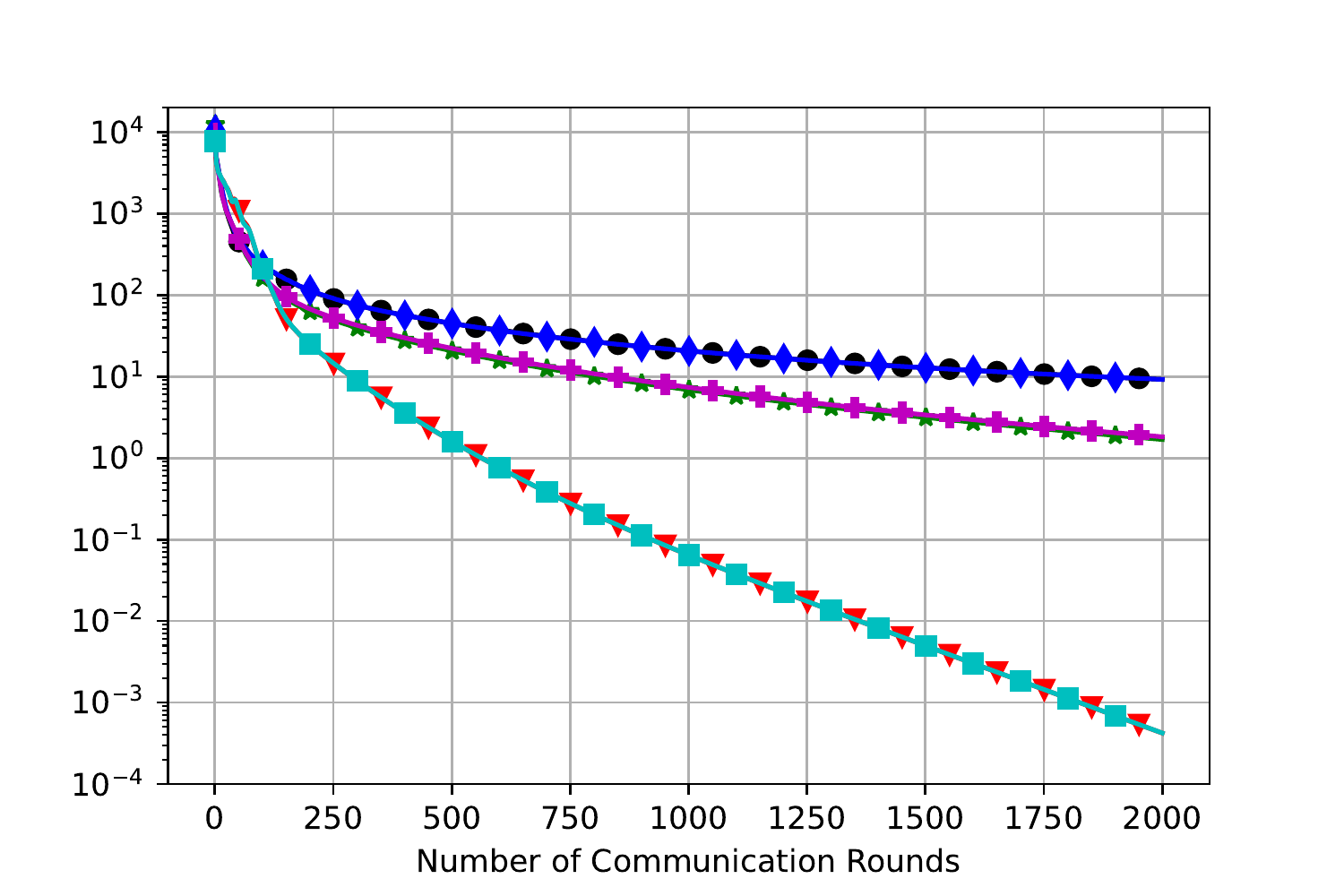} 
\end{subfigure}
\begin{subfigure}{.23\textwidth}
  \centering
\includegraphics[scale=0.3]{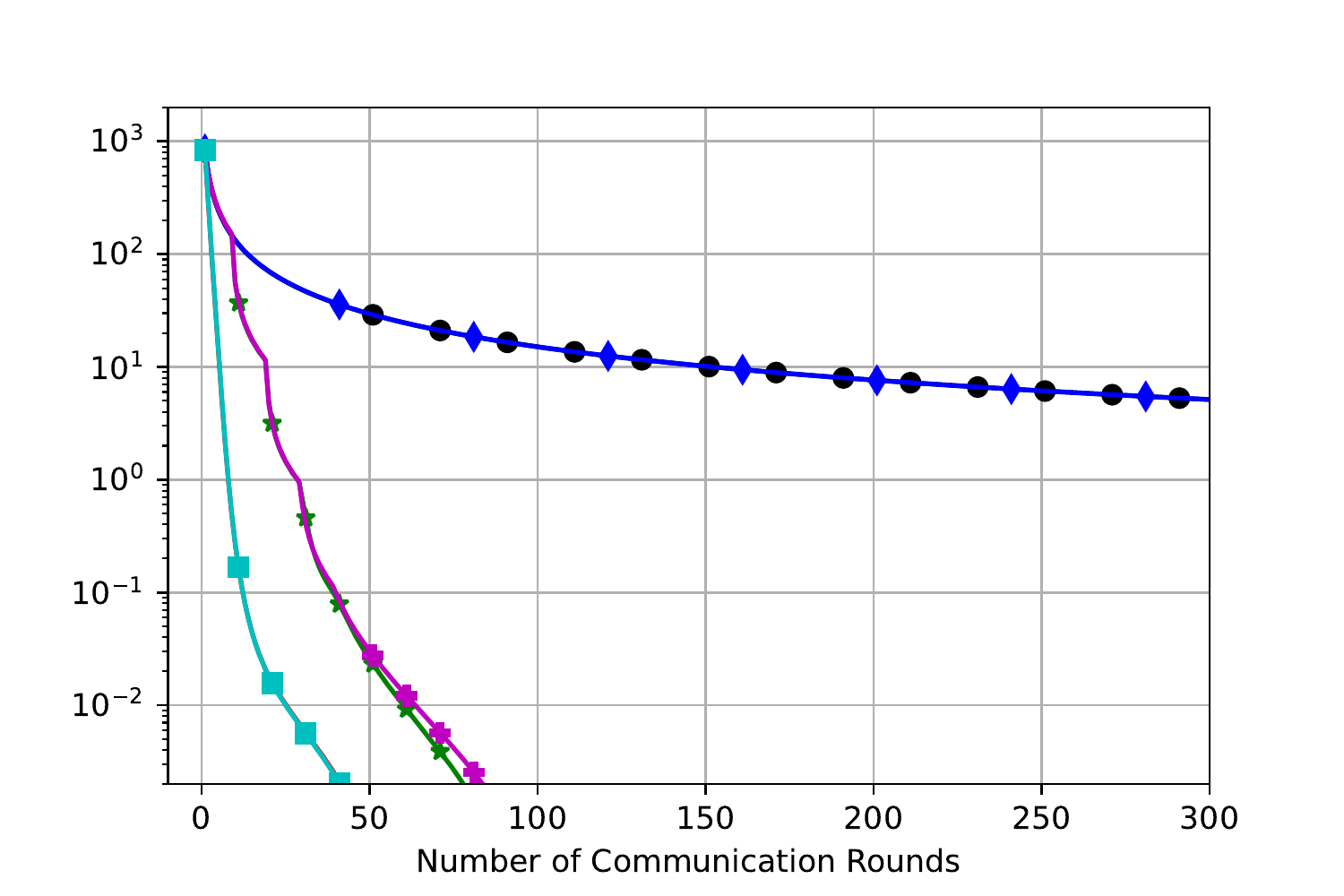} 
\end{subfigure}
\begin{subfigure}{.23\textwidth}
  \centering
\includegraphics[scale=0.3]{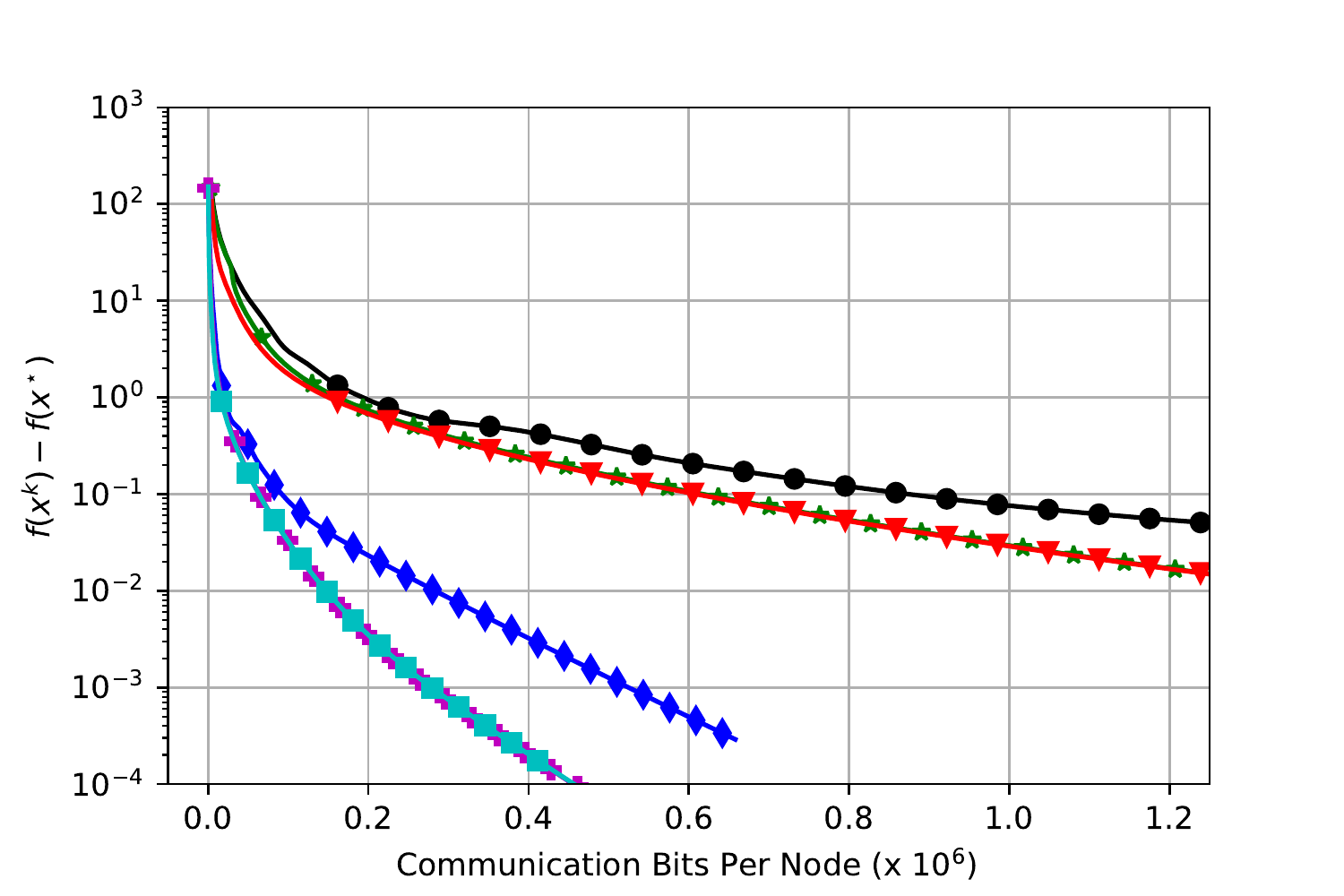} 
  \caption{Dataset: a1a}
\end{subfigure}
\begin{subfigure}{.23\textwidth}
  \centering
\includegraphics[scale=0.3]{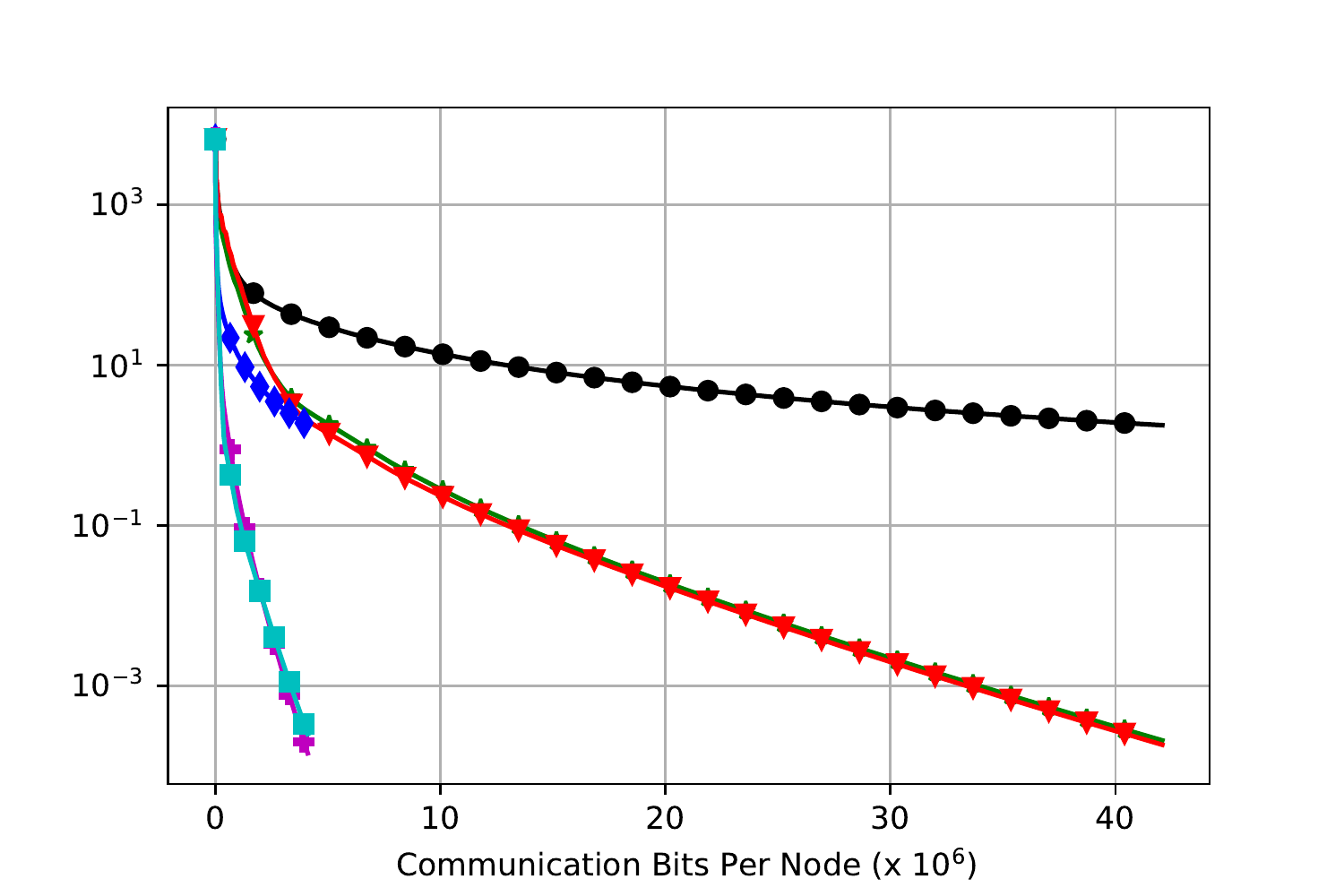} 
  \caption{Dataset: w7a}
\end{subfigure}
\begin{subfigure}{.23\textwidth}
  \centering
\includegraphics[scale=0.3]{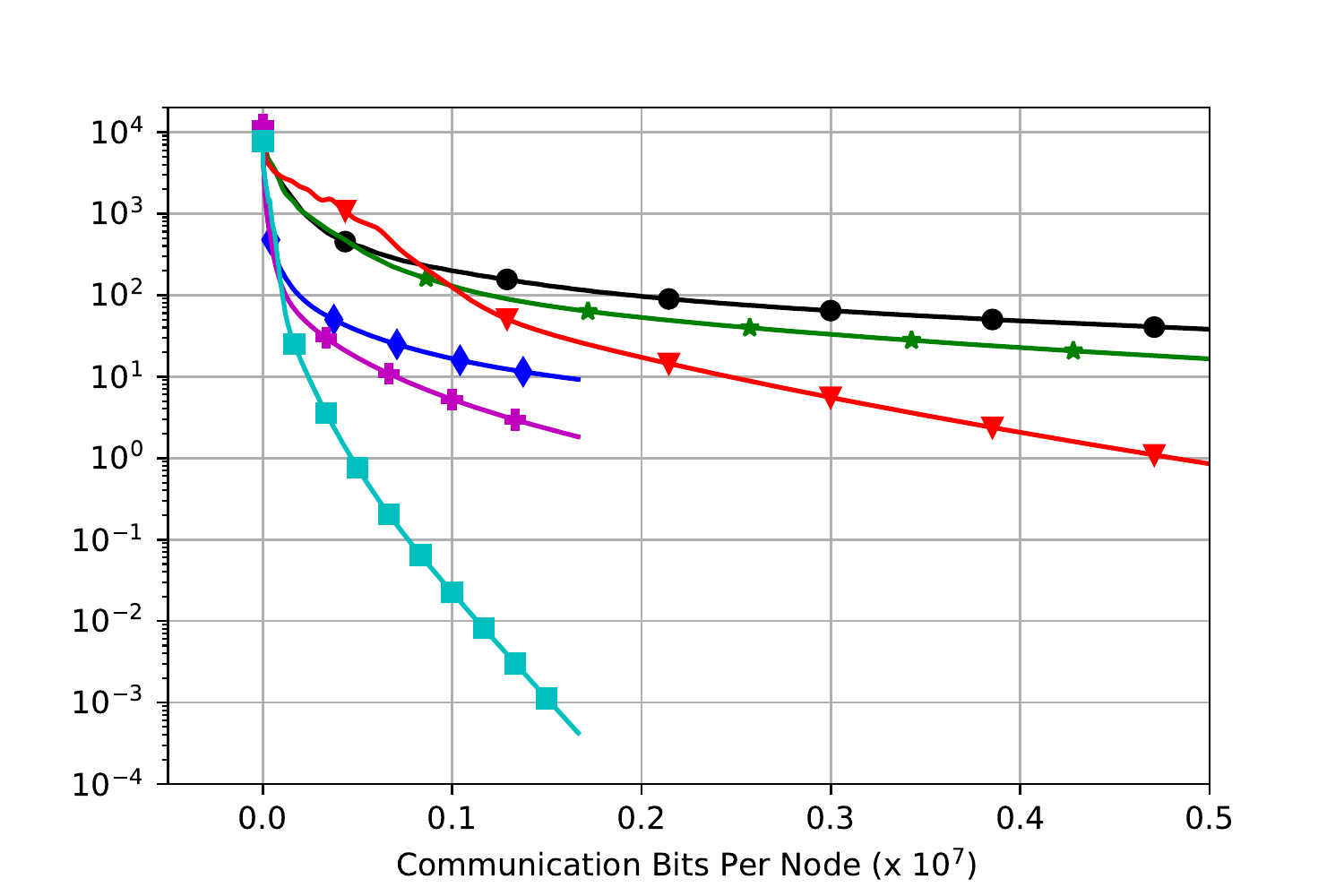} 
  \caption{Dataset: w8a}
\end{subfigure}
\begin{subfigure}{.23\textwidth}
  \centering
\includegraphics[scale=0.3]{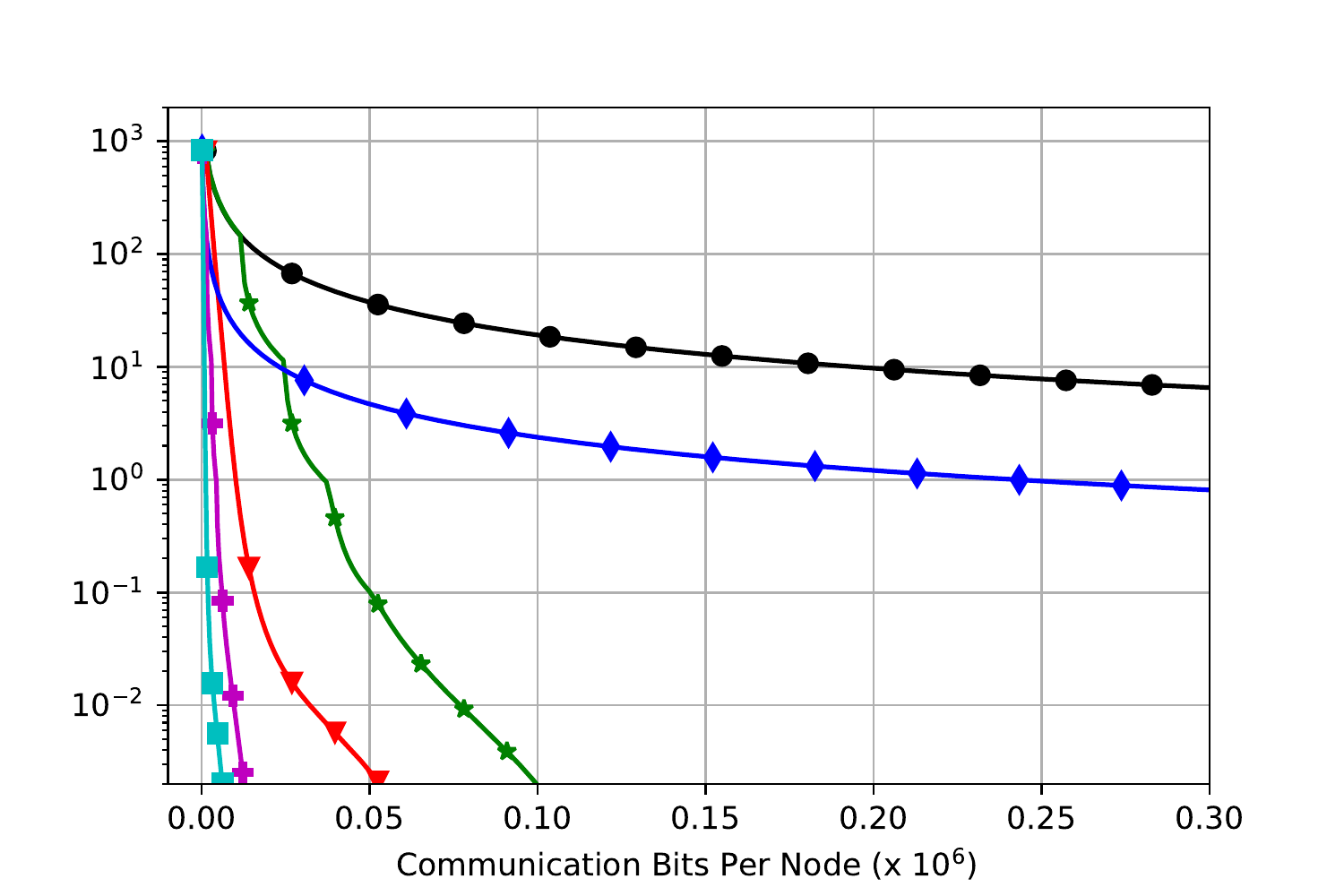} 
  \caption{Dataset: phishing}
\end{subfigure}
\caption{Optimality gap of Q-{\ours} compared to {\ours} in terms of number of communication rounds and number of communication bits per client for different datasets. Q-{\ours} converges as fast as {\ours}, but at significantly less number of transmitted bits.}
\label{fig2}
\end{figure*}

Fig. \ref{fig2} compares Q-{\ours} to {\ours}. As shown in the figure, for a fixed number of communication rounds, Q-{\ours} achieves the same optimality gap as {\ours}. However, when the optimality gap is plotted against the number of communicated bits per client, we can clearly see the significant savings in terms of the number of communicated bits per client that Q-{\ours} achieves compared to {\ours}. For example, for the dataset w8a, Q-{\ours}($r=1$) requires almost $10\times$ less number of transmitted bits compared to {\ours}($r=1$) to achieve the optimality gap of $10^{-3}$. 
%

%
\begin{table}[h]
\caption{Description of the datasets}
\label{table1}
\begin{center}
\begin{tabular}{lcccr}
\toprule
Dataset & $N$ & $m$ & $d$ & $n$ \\
\midrule
a1a    & $1600$ & $160$ & $99$ &  $10$\\
w7a    & $24640$ & $308$ & $263$ &  $80$\\
w8a    & $49700$ & $829$ & $267$ &  $60$\\
phishing& $11040$ & $276$ & $40$ &  $40$\\
\bottomrule
\end{tabular}
\end{center}
\end{table}
%

Finally, we observe, from Figs. \ref{fig1} and \ref{fig2}, that Newton Zero starts at a large number of communicated bits since it requires every client to send the whole Hessian matrix at the first iteration, which consumes a large number of transmitted bits. 

\section{Conclusion and Future Work}
We proposed a novel communication-efficient, and privacy-preserving federated learning framework based on Newton and ADMM methods. Unlike existing approaches, the proposed approach ({\ours}) does not require clients to transmit their Hessian or its compressed version at any iteration. Moreover, the proposed approach ensures privacy by hiding the gradient and the Hessian information. {\ours} achieves the same communication-efficiency of first-order methods per iteration, while enjoying faster convergence and preserving privacy. In particular, the {\it inverse Hessian-gradient product} alternates between updating the {\it inverse Hessian-gradient product} using only one ADMM step at each Newton's iteration, and updating the global model using Newton's method. {{\ours} is proved to follow the inexact Newton directions asymptotically. The non-asymptotic version of the proof of optimality is left to the future scope of this work.} Furthermore, a significant reduction in communication overhead is achieved by utilizing stochastic quantization. Numerical results show the superiority of {\ours} compared to existing methods in terms of communication costs while ensuring privacy. Future works will explore the convergence analysis of the  quantized version of {\ours} (Q-{\ours}), and extend the current framework to fully decentralized topology.

\section{Acknowledgement}
This work is supported by Academy of Finland SMARTER, EU-ICT IntellIoT, EU-CHISTERA LeadingEdge, and CONNECT, Infotech-NOOR and NEGEIN.

\bibliography{main}
\bibliographystyle{plainnat}
\clearpage

\appendix

\input{appendix}

\end{document}

%% file: appendix.tex
\onecolumn
\appendix
\section*{Appendices}

\section{Proof of Lemma \ref{lemma_1}}\label{proof_lemma_1}
At iteration $k$, and given $x^k$, we run the one step ADMM to evaluate the direction $y_i^k$ at each client $i$. From the first order optimality condition of the problem in \eqref{updateLocal}, we can write
\begin{align}\label{feasi0}
 (H_i^k+\alpha I) y_i^k-g_i^k + \lambda_i^{k-1} + \rho (y_i^k - y^{k-1}) =0.
\end{align}
We add and subtract $y^{k}$ in the left hand side in the above expression and utilize the dual udpate  $\lambda_{i}^{k}=\lambda_{i}^{k-1}+ \rho (y_{i}^{k} - y^{k})$ (cf. \eqref{updateDual})  in \eqref{feasi0} to get 
\begin{align}\label{eqo10}
(H_i^k+\alpha I) y_i^k-g_i^k + \lambda_i^k + \rho (y^{k} - y^{k-1})=0.
\end{align}
Let $s^k$ be the dual residual at iteration $k$ defined as 
\begin{equation}\label{dualResidualEq}
s^k=\rho(y^{k}-y^{k-1})
\end{equation}
Therefore, \eqref{eqo10} can be re-written as
\begin{align}\label{eqo20}
(H_i^k+\alpha I) y_i^k-g_i^k + \lambda_i^k + s^k = 0.
\end{align}
Re-arranging the terms, we get
\begin{align}\label{eqo200}
\lambda_i^k + s^k = g_i^k -(H_i^k+\alpha I) y_i^k.
\end{align}
Subtracting ${\lambda_i^\star}^k$ from both sides and using \eqref{dual_feasibility}, we obtain
\begin{align}\label{}
\lambda_i^k + s^k-{\lambda_i^\star}^k = (H_i^k+\alpha I)({y_i^\star}^k-y_i^k)
\end{align}
Multiplying both sides  with $y_i^k-{y_i^\star}^k$, we obtain,
\begin{align}\label{lemma1Result0}
\langle \lambda_i^k + s^k-{\lambda_i^\star}^k, y_i^k-{y_i^\star}^k\rangle = \langle (H_i^k+\alpha I){y_i^\star}^k-y_i^k, y_i^k-{y_i^\star}^k\rangle
\end{align}
Since $H_i^k$ is a positive semidefinite. i.e.,  $\langle H_i^k({y_i^\star}^k-y_i^k), y_i^k-{y_i^\star}^k\rangle \leq 0$, we can re-write
\eqref{lemma1Result0} as
\begin{align}\label{lemma1Result1}
\langle \lambda_i^k + s^k-{\lambda_i^\star}^k, y_i^k-{y^\star}^k\rangle \leq -\alpha \|y_i^k-{y_i^\star}^k\|^2
\end{align}
Using ${y_i^\star}^k={y^\star}^k$, and summing over all clients, and we obtain
\begin{align}\label{lemma1ResultFinal}
\sum_{i=1}^n\langle \lambda_i^k + s^k-{\lambda_i^\star}^k, y_i^k-{y^\star}^k\rangle \leq -\alpha \sum_{i=1}^n\|y_i^k-{y^\star}^k\|^2
\end{align}
That concludes the proof.
\section{Proof of Lemma \ref{lemma_2}}\label{proof_lemma_2}
We start with the statement of Lemma \ref{lemma_1} and multiply both sides of \eqref{lemma1ResultFinal} by $2$ to obtain
\begin{align}\label{main0}
2 A + 2 B \leq  -2\alpha \sum_{i=1}^n\|y_i^k-{y^\star}^k\|^2,
\end{align}
where $A := \sum_{i=1}^n \langle  \lambda_i^k -{\lambda_i^\star}^k, y_i^k-{y^\star}^k\rangle$, and $B := \sum_{i=1}^n \langle  s^k, y_i^k-{y^\star}^k \rangle$. First, we focus on the first term on the left side of \eqref{main0}.
Note that since $\sum_{i=1}^n\lambda_i^k=\sum_{i=1}^n{\lambda_i^\star}^k=\bm{0}$, term $A$ can be written as $A=\sum_{i=1}^n \langle  \lambda_i^k -{\lambda_i^\star}^k, y_i^k\rangle$. Hence, we can write
\begin{align}\label{re}
\nonumber 2 A &= 2 \sum_{i=1}^n \langle  \lambda_i^k -{\lambda_i^\star}^k, y_i^{k}\rangle   \\
\nonumber &\stackrel{\mathrm{(a)}}{=} 2 \sum_{i=1}^n \langle  \lambda_i^{k-1} + \rho y_i^{k}-\rho y^{k} -{\lambda_i^\star}^k, y_i^{k}\rangle \\
\nonumber &= 2 \sum_{i=1}^n \langle  \lambda_i^{k-1}-\rho y^{k} -{\lambda_i^\star}^k, y_i^{k}\rangle + 2 \rho \sum_{i=1}^n \|y_i^{k}\|^2 \\
\nonumber &\stackrel{\mathrm{(b)}}{=} \frac{2}{\rho} \sum_{i=1}^n \langle  \lambda_i^{k-1} -\rho y^{k} -{\lambda_i^\star}^k, \lambda_i^k - \lambda_i^{k-1}+ \rho y^k\rangle+ \rho \sum_{i=1}^n  \|y_i^{k}\|^2 + \rho \sum_{i=1}^n \|y_i^{k}\|^2\\
\nonumber &= \frac{2}{\rho} \sum_{i=1}^n \langle  \lambda_i^{k-1} -\rho y^{k} -{\lambda_i^\star}^k, \lambda_i^k - {\lambda_i^\star}^k + {\lambda_i^\star}^k - \lambda_i^{k-1}+ \rho y^k\rangle + \frac{1}{\rho} \sum_{i=1}^n \|\lambda_i^k - \lambda_i^{k-1} + \rho y^k\|^2 +\rho \sum_{i=1}^n  \|y_i^{k}\|^2 \\
\nonumber &= {- \frac{2}{\rho} \sum_{i=1}^n \|\lambda_i^{k-1} - \rho y^k - {\lambda_i^\star}^k \|^2} + \frac{2}{\rho} \sum_{i=1}^n \langle  \lambda_i^{k-1} -\rho y^{k} -{\lambda_i^\star}^k, \lambda_i^k - {\lambda_i^\star}^k \rangle\\
& + \frac{1}{\rho} \sum_{i=1}^n \|\lambda_i^k - \lambda_i^{k-1} + \rho y^k\|^2 +\rho \sum_{i=1}^n \|y_i^{k}\|^2,
\end{align}
where we used the update of the dual variables, i.e. $\lambda_i^k = \lambda_i^{k-1} + \rho y_i^{k} - \rho y^{k} $ in $\mathrm{(a)}$, and replaced $y_i^{k}$ by $(\lambda_i^k - \lambda_i^{k-1} + \rho y^k)/\rho$ in the first two terms of $\mathrm{(b)}$. Now, we use $\lambda_i^k - \lambda_i^{k-1} = (\lambda_i^k - {\lambda_i^\star}^k)-(\lambda_i^{k-1} - {\lambda_i^\star}^k)$ in the third term of \eqref{re} to write
\begin{align}\label{ee}
\nonumber &\frac{1}{\rho} \sum_{i=1}^n \|\lambda_i^k - \lambda_i^{k-1} + \rho y^k\|^2 \\
\nonumber &=  \frac{1}{\rho} \sum_{i=1}^n \|\lambda_i^k - {\lambda_i^\star}^k - \left(\lambda_i^{k-1} -  {\lambda_i^\star}^k - \rho y^k \right)\|^2\\
&= \frac{1}{\rho} \sum_{i=1}^n \|\lambda_i^k - {\lambda_i^\star}^k \|^2 + \frac{1}{\rho} \sum_{i=1}^n \|\lambda_i^{k-1} -  {\lambda_i^\star}^k - \rho y^k\|^2 - \frac{2}{\rho} \sum_{i=1}^n \langle \lambda_i^{k-1} -  {\lambda_i^\star}^k - \rho y^k  , \lambda_i^k - {\lambda_i^\star}^k\rangle.
\end{align}
Utilizing the expression in \eqref{ee} into \eqref{re}, we get
\begin{align}\label{main1_0}
2 A &= \frac{1}{\rho} \sum_{i=1}^n \|\lambda_i^k - {\lambda_i^\star}^k\|^2 - \frac{1}{\rho} \sum_{i=1}^n \|\lambda_i^{k-1} - {\lambda_i^\star}^k - \rho y^k\|^2 + \rho \sum_{i=1}^n  \|y_i^{k}\|^2\nonumber
\\
 &= \frac{1}{\rho} \sum_{i=1}^n \|\lambda_i^k - {\lambda_i^\star}^k\|^2 - \frac{1}{\rho} \sum_{i=1}^n \|\lambda_i^{k-1} - {\lambda_i^\star}^k\|^2 - \rho n \|y^k\|^2 + \rho \sum_{i=1}^n  \|y_i^{k}\|^2
\end{align}

Next, we tackle the term $B$. Replacing the dual residual $s^{k+1}$ by its definition, we get
\begin{align}
 \nonumber 2 B &= 2 \rho \sum_{i=1}^n \langle y^{k} - y^{k-1}, y_i^k-{y^\star}^k \rangle = 2 \rho \sum_{i=1}^n \langle y^{k} - y^{k-1}, y_i^{k}-y^k + y^{k} - {y^\star}^k\rangle \\
& = 2 \rho \sum_{i=1}^n \langle y^{k} - y^{k-1}, r_i^{k}\rangle + 2 \rho n \langle y^{k}  - y^{k-1}, y^{k}-{y^\star}^k\rangle,
\end{align}
where we utilized the definition $r_i^k=y_i^k-y^k$. 
Using $y^{k}-{y^\star}^k = y^{k} - y^{k-1} + y^{k-1}-{y^\star}^k$, we can write
\begin{align}
2 B = 2 \rho \sum_{i=1}^n \langle y^{k} - y^{k-1}, r_i^{k}\rangle + 2 \rho n \|y^{k} - y^{k-1}\|^2 + 2 \rho n \langle y^{k} - y^{k-1}, y^{k-1}-{y^\star}^k\rangle.
\end{align}
Next, using $y^{k}-y^{k-1} = y^{k} - {y^\star}^k -( y^{k-1}-{y^\star}^k)$, we can write
\begin{align}\label{main02}
\nonumber 2 B&= 2 \rho \sum_{i=1}^n \langle y^{k} - y^{k-1}, r_i^{k}\rangle + 2 \rho n \|y^{k} - y^{k-1}\|^2
 - 2 \rho \| y^{k-1} - {y^\star}^k\|^2 + 2 \rho n \langle y^k-{y^\star}^k, y^{k-1}-{y^\star}^k\rangle \\
\nonumber &= 2 \rho \sum_{i=1}^n \langle y^{k} - y^{k-1}, r_i^{k}\rangle + \rho n \|(y^{k} - {y^\star}^k) - (y^{k-1}-{y^\star}^k)\|^2 + \rho n \|y^{k} - y^{k-1}\|^2
- 2 \rho \| y^{k-1} - {y^\star}^k\|^2\\
&\quad+ 2 \rho n \langle y^k-{y^\star}^k, y^{k-1}-{y^\star}^k\rangle \\
\nonumber &= \rho n \|y^{k} - y^{k-1}\|^2 + \rho n \|y^{k}-{y^\star}^k \|^2-\rho n \|y^{k-1}-{y^\star}^k\|^2 + 2 \rho \sum_{i=1}^n \langle y^{k} - y^{k-1}, r_i^{k}\rangle.
\end{align}
Using the definition of $r_i^{k}$ and the equality in \eqref{updatePS3}, we have
\begin{align}\label{eqq}
\sum_{i=1}^n \langle y^{k} - y^{k-1}, r_i^{k}\rangle
 &= \langle y^{k} - y^{k-1}, \sum_{i=1}^n r_i^{k}\rangle = \langle y^{k} - y^{k-1}, \sum_{i=1}^n y_i^{k} - n y^{k}\rangle = 0.
\end{align}
Hence, from \eqref{main02}, we obtain
\begin{align}\label{main2}
2 B &= \rho n \|y^{k} - y^{k-1}\|^2 + \rho n \|y^{k}-{y^\star}^k \|^2-\rho n \|y^{k-1}-{y^\star}^k\|^2.
\end{align}
Substituting \eqref{main1_0} and \eqref{main02} into \eqref{main0} and rearrange terms, we obtain, 
\begin{align}\label{combinedAB}
&\frac{1}{\rho} \sum_{i=1}^n \|\lambda_i^k- {\lambda_i^\star}^k\|^2  + \rho n \|y^{k}-{y^\star}^k \|^2
\\
&
 \leq \frac{1}{\rho} \sum_{i=1}^n \|\lambda_i^{k-1} - {\lambda_i^\star}^k\|^2+\rho n \|y^{k-1}-{y^\star}^k\|^2-\rho n \|y^{k} - y^{k-1}\|^2+ \rho n \|y^k\|^2 - \rho \sum_{i=1}^n  \|y_i^{k}\|^2-2\alpha \sum_{i=1}^n\|y_i^k-{y^\star}^k\|^2\nonumber
\end{align}
Using $n\sum_{i=1}^n  \|y_i^{k}\|^2 \geq   \|\sum_{i=1}^n y_i^{k}\|^2$, we can easily show that $\rho \sum_{i=1}^n  \|y_i^{k}\|^2 \geq \rho n \|y^k\|^2$, so both terms cancel from the RHS of \eqref{combinedAB}, and we obtain
\begin{align}\label{}
&\frac{1}{\rho} \sum_{i=1}^n \|\lambda_i^k- {\lambda_i^\star}^k\|^2  + \rho n \|y^{k}-{y^\star}^k \|^2
\nonumber\\
&
 \leq \frac{1}{\rho} \sum_{i=1}^n \|\lambda_i^{k-1}- {\lambda_i^\star}^k\|^2+\rho n \|y^{k-1}-{y^\star}^k\|^2-\rho n \|y^{k} - y^{k-1}\|^2-2\alpha \sum_{i=1}^n \|y_i^k-{y^\star}^k\|^2.
\end{align}
Adding and subtracting $\lambda_i^{k-1}$, we get
\begin{align}\label{}
&\frac{1}{\rho} \sum_{i=1}^n \|\lambda_i^k- {\lambda_i^\star}^k\|^2  + \rho n \|y^{k}-{y^\star}^k \|^2
\nonumber\\
&
 \leq \frac{1}{\rho} \sum_{i=1}^n \|\lambda_i^{k-1} -{\lambda_i^\star}^{k-1}+{\lambda_i^\star}^{k-1}- {\lambda_i^\star}^k\|^2+\rho n \|y^{k-1}-{y^\star}^k\|^2-\rho n \|y^{k} - y^{k-1}\|^2-2\alpha \sum_{=1}^n \|y_i^k-{y^\star}^k\|^2
\end{align}
Using $\|a+b\|^2 \leq 2\|a\|^2+2\|b\|^2$, the following holds
\begin{align}\label{}
&\frac{1}{\rho} \sum_{i=1}^n \|\lambda_i^k- {\lambda_i^\star}^k\|^2  + \rho n \|y^{k}-{y^\star}^k \|^2
\\
&
 \leq \frac{2}{\rho} \sum_{i=1}^n \|\lambda_i^{k-1} -{\lambda_i^\star}^{k-1}\|^2+\frac{2}{\rho} \sum_{i=1}^n\|{\lambda_i^\star}^{k-1}- {\lambda_i^\star}^k\|^2+\rho n \|y^{k-1}-{y^\star}^k\|^2-\rho n \|y^{k} - y^{k-1}\|^2-2\alpha \sum_{=1}^n \|y_i^k-{y^\star}^k\|^2\nonumber.
\end{align}
Using ${\lambda_i^\star}^k = g_i^{k} - (H_i^{k}+\alpha^{k} I) {y_i^\star}^{k}$, we get
\begin{align}\label{eq56}
&\frac{1}{\rho} \sum_{i=1}^n \|\lambda_i^k- {\lambda_i^\star}^k\|^2  + \rho n \|y^{k}-{y^\star}^k \|^2
\nonumber\\
&
 \leq \frac{2}{\rho} \sum_{i=1}^n \|\lambda_i^{k-1} -{\lambda_i^\star}^{k-1}\|^2+\frac{2}{\rho} \sum_{i=1}^n\|g_i^{k-1} - (H_i^{k-1}+\alpha^{k-1} I) {y_i^\star}^{k-1}-g_i^{k} + (H_i^{k}+\alpha^{k} I) {y_i^\star}^{k}\|^2\nonumber\\
 &\quad+\rho n \|y^{k-1}-{y^\star}^k\|^2-\rho n \|y^{k} - y^{k-1}\|^2-2\alpha \sum_{=1}^n \|y_i^k-{y^\star}^k\|^2.
\end{align}
Using the definition of $Q_i(\cdot,\cdot)$, we write \eqref{eq56} as
\begin{align}\label{eq57}
&\frac{1}{\rho} \sum_{i=1}^n \|\lambda_i^k- {\lambda_i^\star}^k\|^2  + \rho n \|y^{k}-{y^\star}^k \|^2
\nonumber\\
&
 \leq \frac{2}{\rho} \sum_{i=1}^n \|\lambda_i^{k-1} -{\lambda_i^\star}^{k-1}\|^2+\frac{2}{\rho} \sum_{i=1}^n\|\nabla Q_i(x^{k-1}, {y_i^\star}^{k-1})-\nabla Q_i(x^k, {y_i^\star}^{k})\|^2\nonumber\\
 &\quad+\rho n \|y^{k-1}-{y^\star}^k\|^2-\rho n \|y^{k} - y^{k-1}\|^2-2\alpha \sum_{=1}^n \|y_i^k-{y^\star}^k\|^2.
\end{align}
Using the lipschitz continuity of $\nabla Q$, we obtain
\begin{align}\label{eq58}
&\frac{1}{\rho} \sum_{i=1}^n \|\lambda_i^k- {\lambda_i^\star}^k\|^2  + \rho n \|y^{k}-{y^\star}^k \|^2
\nonumber\\
\nonumber &
 \leq \frac{2}{\rho} \sum_{i=1}^n \|\lambda_i^{k-1} -{\lambda_i^\star}^{k-1}\|^2+\frac{2L_q^2}{\rho} \sum_{i=1}^n\| {y_i^\star}^{k-1}- {y_i^\star}^{k}\|^2+\rho n \|y^{k-1}-{y^\star}^k\|^2\\
 &-\rho n \|y^{k} - y^{k-1}\|^2-2\alpha \sum_{=1}^n \|y_i^k-{y^\star}^k\|^2.
\end{align}
%
%
Equivalently, we can write \eqref{eq58} as
\begin{align}\label{}
&\frac{1}{\rho} \sum_{i=1}^n \|\lambda_i^k- {\lambda_i^\star}^k\|^2  + \rho n \|y^{k}-{y^\star}^k \|^2
\nonumber\\
&
 \leq \frac{1}{\rho} \sum_{i=1}^n \|\lambda_i^{k-1} -{\lambda_i^\star}^{k-1}\|^2+\frac{1}{\rho} \sum_{i=1}^n \|g_i^{k-1} -s^{k-1}- (H_i^{k-1}+\alpha^{k-1} I) y_i^{k-1}-g_i^{k-1} + (H_i^{k-1}+\alpha^{k-1} I) {y_i^\star}^{k-1}\|^2\nonumber\\
 &\quad+\frac{2L_q^2}{\rho} \sum_{i=1}^n\| {y_i^\star}^{k-1}- {y_i^\star}^{k}\|^2+\rho n \|y^{k-1}-{y^\star}^k\|^2-\rho n \|y^{k} - y^{k-1}\|^2-2\alpha \sum_{=1}^n \|y_i^k-{y^\star}^k\|^2.
\end{align}
Using the inequality $\|a+b\|^2 \leq 2\|a\|^2+2\|b\|^2$ and the expression of $\nabla Q$, we obtain
\begin{align}\label{}
&\frac{1}{\rho} \sum_{i=1}^n \|\lambda_i^k- {\lambda_i^\star}^k\|^2  + \rho n \|y^{k}-{y^\star}^k \|^2
\nonumber\\
&
 \leq \frac{1}{\rho} \sum_{i=1}^n \|\lambda_i^{k-1} -{\lambda_i^\star}^{k-1}\|^2+\frac{2}{\rho} \sum_{i=1}^n \|\nabla Q_i(x^{k-1},y_i^{k-1})-\nabla Q_i(x^{k-1},{y_i^\star}^{k-1})\|^2+\frac{2n}{\rho}\|s^{k-1}\|^2\nonumber\\
 \nonumber &\quad+\frac{2L_q^2}{\rho} \sum_{i=1}^n\| {y_i^\star}^{k-1}- {y_i^\star}^{k}\|^2+\rho n \|y^{k-1}-{y^\star}^k\|^2-\rho n \|y^{k} - y^{k-1}\|^2-2\alpha \sum_{=1}^n \|y_i^k-{y^\star}^k\|^2\\
&
 \leq \frac{1}{\rho} \sum_{i=1}^n \|\lambda_i^{k-1} -{\lambda_i^\star}^{k-1}\|^2+\frac{2 L_q^2}{\rho}  \sum_{i=1}^n\|y_i^{k-1}-{y^\star}^{k-1}\|^2+\frac{2n}{\rho}\|s^{k-1}\|^2+\frac{2 L_q^2}{\rho}\sum_{i=1}^n\| {y^\star}^{k-1}- {y^\star}^{k}\|^2\nonumber\\
 &\quad+\rho n \|y^{k-1}-{y^\star}^k\|^2-\rho n \|y^{k} - y^{k-1}\|^2-2\alpha \sum_{=1}^n \|y_i^k-{y^\star}^k\|^2.
\end{align}
%
%
Replacing $s^{k-1}$ by its definition, we get
\begin{align}\label{}
&\frac{1}{\rho} \sum_{i=1}^n \|\lambda_i^k- {\lambda_i^\star}^k\|^2  + \rho n \|y^{k}-{y^\star}^k \|^2
\nonumber\\
&
 \leq \frac{1}{\rho} \sum_{i=1}^n \|\lambda_i^{k-1} -{\lambda_i^\star}^{k-1}\|^2+\frac{2 L_q^2}{\rho}  \sum_{i=1}^n\|y_i^{k-1}-{y^\star}^{k-1}\|^2+2n\rho\|y^{k-1}-y^{k-2}\|^2+\frac{2 L_q^2n}{\rho}\| {y^\star}^{k-1}- {y^\star}^{k}\|^2\nonumber\\
 &\quad+\rho n \|y^{k-1}-{y^\star}^{k}\|^2-\rho n \|y^{k} - y^{k-1}\|^2-2\alpha \sum_{i=1}^n \|y_i^{k}-{y^\star}^k\|^2.
\end{align}
%
%
%
Let $\alpha = \alpha_1 + \alpha_2$. Adding and subtracting $y^k$ in $\|y^{k-1}-{y^\star}^k\|^2$ and using $\|a+b\|^2 \leq 2\|a\|^2+2\|b\|^2$, we obtain
\begin{align}\label{eq62}
&\frac{1}{\rho} \sum_{i=1}^n \|\lambda_i^k- {\lambda_i^\star}^k\|^2  + \rho n \|y^{k}-{y^\star}^k \|^2
\nonumber\\
&
 \leq \frac{1}{\rho} \sum_{i=1}^n \|\lambda_i^{k-1} -{\lambda_i^\star}^{k-1}\|^2+\frac{2 L_q^2}{\rho} \sum_{i=1}^n\|y_i^{k-1}-{y^\star}^{k-1}\|^2+2\rho n\|y^{k-1}-y^{k-2}\|^2+\frac{2 L_q^2n}{\rho}\| {y^\star}^{k-1}- {y^\star}^{k}\|^2\nonumber\\
 &\quad+2\rho n \|y^k-y^{k-1}\|^2+2\rho n\|y^k-{y^\star}^k\|^2-\rho n \|y^{k} - y^{k-1}\|^2-2\alpha_1 \sum_{i=1}^n \|y_i^{k}-{y^\star}^k\|^2-2\alpha_2 \sum_{i=1}^n \|y_i^{k}-{y^\star}^k\|^2.
\end{align}
 Using $\sum_{i=1}^n \|y_i^{k}-{y^\star}^k\|^2 \geq  \|\sum_{i=1}^n(y_i^{k}-{y^\star}^k)\|^2=  n \|y^{k}-{y^\star}^k\|^2 $, we can rewrite \eqref{eq62} as
 \begin{align}\label{}
&\frac{1}{\rho} \sum_{i=1}^n \|\lambda_i^k- {\lambda_i^\star}^k\|^2  + \rho n \|y^{k}-{y^\star}^k \|^2
\nonumber\\
&
 \leq \frac{1}{\rho} \sum_{i=1}^n \|\lambda_i^{k-1} -{\lambda_i^\star}^{k-1}\|^2+\frac{2 L_q^2}{\rho}\sum_{i=1}^n \|y_i^{k-1}-{y^\star}^{k-1}\|^2+2\rho n\|y^{k-1}-y^{k-2}\|^2+\frac{2 L_q^2n}{\rho}\| {y^\star}^{k-1}- {y^\star}^{k}\|^2\nonumber\\
 &\quad+2\rho n \|y^k-y^{k-1}\|^2+2\rho n\|y^k-{y^\star}^k\|^2-\rho n \|y^{k} - y^{k-1}\|^2-2\alpha_1 n \|y^{k}-{y^\star}^k\|^2-2\alpha_2 \sum_{i=1}^n \|y_i^{k}-{y^\star}^k\|^2.
\end{align}
 Assuming $\|y^{k-1}-y^k\|^2 \leq \|y^{k}-{y^\star}^k\|^2$ which holds for sufficiently large $\rho$, and choosing $\alpha_1 \geq 2.5\rho$, we get
\begin{align}\label{}
&\frac{1}{\rho} \sum_{i=1}^n \|\lambda_i^k- {\lambda_i^\star}^k\|^2  + \rho n \|y^{k}-{y^\star}^k \|^2
\nonumber\\
&
 \leq \frac{1}{\rho} \sum_{i=1}^n \|\lambda_i^{k-1} -{\lambda_i^\star}^{k-1}\|^2+\frac{2 L_q^2}{\rho}\sum_{i=1}^n \|y_i^{k-1}-{y^\star}^{k-1}\|^2+2\rho n\|y^{k-1}-y^{k-2}\|^2+\frac{2 L_q^2n}{\rho}\| {y^\star}^{k-1}- {y^\star}^{k}\|^2\nonumber\\
 &\quad-2\rho n \|y^{k} - y^{k-1}\|^2-2\alpha_2 \sum_{i=1}^n \|y^{k}-{y^\star}^k\|^2.
\end{align}
Re-arranging the terms, we get
\begin{align}\label{}
&\frac{1}{\rho} \sum_{i=1}^n \|\lambda_i^k- {\lambda_i^\star}^k\|^2  + \rho n \|y^{k}-{y^\star}^k \|^2+2\rho n \|y^{k} - y^{k-1}\|^2
\nonumber\\
&
 \leq \frac{1}{\rho} \sum_{i=1}^n \|\lambda_i^{k-1} -{\lambda_i^\star}^{k-1}\|^2+\frac{2 L_q^2}{\rho}\sum_{i=1}^n \|y_i^{k-1}-{y^\star}^{k-1}\|^2+2\rho n\|y^{k-1}-y^{k-2}\|^2\nonumber\\
 &\quad+\frac{2 L_q^2n}{\rho}\| {y^\star}^{k-1}- {y^\star}^{k}\|^2-2\alpha_2 \sum_{i=1}^n \|y_i^{k}-{y^\star}^k\|^2.
\end{align}
Adding and subtracting $y^{k-1}$ in $\|{y^\star}^{k-1}-{y^\star}^{k}\|^2$, we get
\begin{align}\label{}
&\frac{1}{\rho} \sum_{i=1}^n \|\lambda_i^k- {\lambda_i^\star}^k\|^2  + \rho n \|y^{k}-{y^\star}^k \|^2+2\rho n \|y^{k} - y^{k-1}\|^2
\nonumber\\
&
 \leq \frac{1}{\rho} \sum_{i=1}^n \|\lambda_i^{k-1} -{\lambda_i^\star}^{k-1}\|^2+\frac{2 L_q^2 }{\rho}\sum_{i=1}^n \|y_i^{k-1}-{y^\star}^{k-1}\|^2+2\rho n\|y^{k-1}-y^{k-2}\|^2+\frac{4 L_q^2n}{\rho}\| {y^\star}^{k-1}-y^{k-1}\|^2\nonumber\\
 &\quad+\frac{4 L_q^2n}{\rho}\|y^{k-1}- {y^\star}^{k}\|^2-2\alpha_2 \sum_{i=1}^n \|y_i^{k}-{y^\star}^k\|^2.
\end{align}
Next, let $\alpha_2 = \alpha_3 + \alpha_4$. Adding and subtracting $y^{k}$ in $\|y^{k-1}-{y^\star}^{k}\|^2$, we get
\begin{align}\label{}
&\frac{1}{\rho} \sum_{i=1}^n \|\lambda_i^k- {\lambda_i^\star}^k\|^2  + \rho n \|y^{k}-{y^\star}^k \|^2+2\rho n \|y^{k} - y^{k-1}\|^2
\nonumber\\
&
 \leq \frac{1}{\rho} \sum_{i=1}^n \|\lambda_i^{k-1} -{\lambda_i^\star}^{k-1}\|^2+\frac{2 L_q^2}{\rho}\sum_{i=1}^n \|y_i^{k-1}-{y^\star}^{k-1}\|^2+2\rho n\|y^{k-1}-y^{k-2}\|^2+\frac{4 L_q^2n}{\rho}\| {y^\star}^{k-1}-y^{k-1}\|^2\nonumber\\
 &\quad+\frac{8 L_q^2n}{\rho}\|y^{k-1}-y^k\|^2+\frac{8 L_q^2n}{\rho}\| y^k-{y^\star}^{k}\|^2-2\alpha_3 n \|y^{k}-{y^\star}^k\|^2-2\alpha_4 \sum_{i=1}^n \|y_i^{k}-{y^\star}^k\|^2.
\end{align}
Choosing $\alpha_3 \geq \frac{8 L_q^2n}{\rho}$, and using the assumption $\|y^{k-1}-y^k\|^2 \leq \|y^{k}-{y^\star}^k\|^2$, we get
\begin{align}\label{eq68}
&\frac{1}{\rho} \sum_{i=1}^n \|\lambda_i^k- {\lambda_i^\star}^k\|^2  + \rho n \|y^{k}-{y^\star}^k \|^2+2\rho n \|y^{k} - y^{k-1}\|^2
\nonumber\\
&
 \leq \frac{1}{\rho} \sum_{i=1}^n \|\lambda_i^{k-1} -{\lambda_i^\star}^{k-1}\|^2+\frac{2 L_q^2}{\rho}\sum_{i=1}^n \|y_i^{k-1}-{y^\star}^{k-1}\|^2+2\rho n\|y^{k-1}-y^{k-2}\|^2+\frac{4 L_q^2n}{\rho}\| {y^\star}^{k-1}-y^{k-1}\|^2\nonumber\\
 &\quad-2\alpha_4 \sum_{i=1}^n \|y_i^{k}-{y^\star}^k\|^2.
\end{align}
Let $\alpha_4=\alpha_5+\alpha_6$, we write \eqref{eq68} as
\begin{align}\label{}
&\frac{1}{\rho} \sum_{i=1}^n \|\lambda_i^k- {\lambda_i^\star}^k\|^2  + \rho n \|y^{k}-{y^\star}^k \|^2+2\rho n \|y^{k} - y^{k-1}\|^2
\nonumber\\
&
 \leq \frac{1}{\rho} \sum_{i=1}^n \|\lambda_i^{k-1} -{\lambda_i^\star}^{k-1}\|^2+\frac{2 L_q^2}{\rho}\sum_{i=1}^n \|y_i^{k-1}-{y^\star}^{k-1}\|^2+2\rho n\|y^{k-1}-y^{k-2}\|^2+\frac{4 L_q^2n}{\rho}\|y^{k-1}-{y^\star}^{k-1}\|^2\nonumber\\
 &\quad-2\alpha_5 \sum_{i=1}^n \|y_i^{k}-{y^\star}^k\|^2-2\alpha_6 \sum_{i=1}^n \|y_i^{k}-{y^\star}^k\|^2.
\end{align}
Re-arrange terms, we obtain
\begin{align}\label{eq70}
&\frac{1}{\rho} \sum_{i=1}^n \|\lambda_i^k- {\lambda_i^\star}^k\|^2+2\alpha_5 \sum_{i=1}^n \|y_i^{k}-{y^\star}^k\|^2  + \rho n \|y^{k}-{y^\star}^k \|^2+2\rho n \|y^{k} - y^{k-1}\|^2
\nonumber\\
&
 \leq \frac{1}{\rho} \sum_{i=1}^n \|\lambda_i^{k-1} -{\lambda_i^\star}^{k-1}\|^2+\frac{2 L_q^2}{\rho}\sum_{i=1}^n \|y_i^{k-1}-{y^\star}^{k-1}\|^2+\frac{4 L_q^2n}{\rho}\|y^{k-1}-{y^\star}^{k-1}\|^2+2\rho n\|y^{k-1}-y^{k-2}\|^2\nonumber\\
 &\quad-2\alpha_6 \sum_{i=1}^n \|y_i^{k}-{y^\star}^k\|^2.
\end{align}
Let $\beta_1=\alpha_5$, and $\beta_2=\alpha_6$, Hence, we can write \eqref{eq70} as
\begin{align}\label{lemma2Result}
&\frac{1}{\rho} \sum_{i=1}^n \|\lambda_i^k- {\lambda_i^\star}^k\|^2+2\beta_1 \sum_{i=1}^n \|y_i^{k}-{y^\star}^k\|^2  + \rho n \|y^{k}-{y^\star}^k \|^2+2\rho n \|y^{k} - y^{k-1}\|^2
\nonumber\\
&
 \leq \frac{1}{\rho} \sum_{i=1}^n \|\lambda_i^{k-1} -{\lambda_i^\star}^{k-1}\|^2+\frac{2 L_q^2}{\rho}\sum_{i=1}^n \|y_i^{k-1}-{y^\star}^{k-1}\|^2+\frac{4 L_q^2n}{\rho}\|y^{k-1}-{y^\star}^{k-1}\|^2+2\rho n\|y^{k-1}-y^{k-2}\|^2\nonumber\\
 &\quad-2\beta_2 \sum_{i=1}^n \|y_i^{k}-{y^\star}^k\|^2.
\end{align}
Note that,
\begin{align}
\alpha=&\alpha_1+\alpha_2\nonumber\\
=&\alpha_1+\alpha_3+\beta_1+\beta_2\nonumber\\
\geq& 2.5 \rho + \frac{8L_q^2n}{\rho}+\beta_1+\beta_2
\end{align}
Hence,
%
$\beta_1+\beta_2\leq \alpha - 2.5 \rho - \frac{8L_q^2n}{\rho}$. 
%
That concludes the proof.

\section{Proof of Theorem \ref{theorem_1}}\label{proof_of_theorem}
From the statement of Lemma \ref{lemma_2}, we can write 
\begin{align}\label{lemma2Result2}
\frac{1}{\rho}& \sum_{i=1}^n \|\lambda_i^k- {\lambda_i^\star}^k\|^2+2\beta_1 \sum_{i=1}^n \|y_i^{k}-{y^\star}^k\|^2  + \rho n \|y^{k}-{y^\star}^k \|^2+2\rho n \|y^{k} - y^{k-1}\|^2
\nonumber\\
&
 \leq \frac{1}{\rho} \sum_{i=1}^n \|\lambda_i^{k-1} -{\lambda_i^\star}^{k-1}\|^2+\frac{2 L_q^2}{\rho}\sum_{i=1}^n \|y_i^{k-1}-{y^\star}^{k-1}\|^2
+\frac{4 L_q^2n}{\rho}\|y^{k-1}-{y^\star}^{k-1}\|^2+2\rho n\|y^{k-1}-y^{k-2}\|^2
\nonumber
\\
&\quad -2\beta_2 \sum_{i=1}^n \|y_i^{k}-{y^\star}^k\|^2.
\end{align}
Choosing $\alpha$ that satisfies \eqref{alphaCondition} for $\beta_1\geq \frac{L_q^2}{\rho}$, $\rho\geq 2L_q$, and $\beta_2 > 0$ and from the definition of Lyapunov function $V_k$ \eqref{Lyapunov}, we can write
\begin{align}\label{them1inq11}
V^k
&\leq V^{k-1}-\left(2\beta_2 \sum_{i=1}^n \|y_i^{k}-{y^\star}^k\|^2\right).
\end{align}
Therefore, $V^k$ decreases in each iteration $k$. 
Consequently, since $V_k\geq 0$, this implies that every term in $V^k$ goes to $0$, i.e. $\lambda_i^k \rightarrow {\lambda_i^\star}^k$, $y_i^{k} \rightarrow {y^\star}^k$, and $y^{k} \rightarrow {y^\star}^k$, 
which completes the proof.